\documentclass{article}

\PassOptionsToPackage{numbers, compress}{natbib}

\usepackage[final]{neurips_2025}

\usepackage[utf8]{inputenc} 
\usepackage[T1]{fontenc}    
\usepackage{hyperref}       
\usepackage{url}            
\usepackage{booktabs}       
\usepackage{amsfonts}       
\usepackage{nicefrac}       
\usepackage{microtype}      
\usepackage{xcolor}         
\hypersetup{
    colorlinks = true,
    citecolor = {green!45!black},
    linkcolor= {red!45!black},
    urlcolor= {magenta!85!black},
}

\usepackage{graphicx}
\graphicspath{{./figures/}{./assets/figures/}}

\usepackage{wrapfig}
\usepackage{multirow}
\usepackage{enumitem}

\usepackage{amsthm}
\usepackage{tikz-cd}  
\usepackage{subcaption}
\usetikzlibrary{positioning, shapes.geometric, fit} 
\usepackage{booktabs}

\usepackage[table]{xcolor}

\newcommand\scalemath[2]{\scalebox{#1}{\mbox{\ensuremath{\displaystyle #2}}}}

\makeatletter
\newcommand{\ifonecol}[2]{%
  \if@twocolumn
    #2%
  \else
    #1%
  \fi
}
\makeatother

\usepackage{thmtools, thm-restate}
\theoremstyle{plain}

\newtheorem{theorem}{Theorem}[section]
\newtheorem{proposition}[theorem]{Proposition}
\newtheorem{lemma}[theorem]{Lemma}

\theoremstyle{definition}

\newtheorem{assumption}[theorem]{Assumption}
\theoremstyle{remark}

\newcommand{\appref}[1]{\hyperref[#1]{Appendix~\ref*{#1}}}

\newcommand{\autorefp}[1]{(\autoref{#1})}
\newcommand{\autorefseq}[2]{\autoref{#1}--\ref{#2}}

\usepackage[acronym, nohypertypes={acronym}]{glossaries}
\glsdisablehyper

\newacronym{gnn}{GNN}{graph neural network}
\newacronym{stgnn}{STGNN}{spatiotemporal graph neural network}
\newacronym{mlp}{MLP}{multilayer perceptron}
\newacronym{cnn}{CNN}{convolutional neural network}
\newacronym{tcn}{TCN}{temporal convolutional network}
\newacronym{mptcn}{MPTCN}{message-passing temporal convolutional network}
\newacronym{rnn}{RNN}{recurrent neural network}
\newacronym{gru}{GRU}{Gated Recurrent Unit}
\newacronym{lstm}{LSTM}{Long Short-Term Memory}
\newacronym{mprnn}{MPRNN}{message-passing recurrent neural network}
\newacronym{mpgru}{MPGRU}{message-passing GRU}

\newacronym{mp}{MP}{message passing}
\newacronym{mpnn}{MPNN}{message-passing neural network}
\newacronym{tmp}{TMP}{temporal message-passing}
\newacronym{smp}{SMP}{spatial message-passing}
\newacronym{stmp}{STMP}{spatiotemporal message-passing}

\newacronym{mae}{MAE}{mean absolute error}
\newacronym{mse}{MSE}{mean squared error}

\newacronym{sn}{SN}{sensor network}
\newacronym{iot}{IoT}{Internet of Things}
\newacronym{mso}{MSO}{Multiple Superimposed Oscillators}

\newacronym{tts}{TTS}{\emph{time-then-space}}
\newacronym{tas}{T\&S}{\emph{time-and-space}}

\newacronym{mcar}{MCAR}{missing completely at random}
\newacronym{mar}{MAR}{missing at random}
\newacronym{mnar}{MNAR}{missing not at random}

\newglossaryentry{tmpd}{name=\textsc{TMP-D},description=}
\newglossaryentry{smpd}{name=\textsc{SMP-D},description=}

\newacronym{model}{???}{???}

\newglossaryentry{ttsimp}{name=TTS-IMP,description=}
\newglossaryentry{ttsamp}{name=TTS-AMP,description=}
\newglossaryentry{tasimp}{name=T\&S-IMP,description=}
\newglossaryentry{tasamp}{name=T\&S-AMP,description=}

\newglossaryentry{fcrnn}{name=FC-RNN,description=}
\newglossaryentry{fcgru}{name=FCGRU,description=}
\newglossaryentry{grud}{name=GRU-D,description=}
\newglossaryentry{fcgrud}{name=FCGRU-D,description=}
\newglossaryentry{grui}{name=GRU-I,description=}
\newglossaryentry{grin}{name=GRIN-P,description=}
\newglossaryentry{dcrnn}{name=DCRNN,description=}
\newglossaryentry{agcrn}{name=AGCRN,description=}
\newacronym{gwnet}{GWNet}{Graph WaveNet}

\newglossaryentry{la}{name=METR-LA,description=}
\newglossaryentry{bay}{name=PEMS-BAY,description=}
\newglossaryentry{air}{name=AQI,description=}
\newglossaryentry{pvus}{name=PV-US,description=}
\newglossaryentry{engrad}{name=EngRAD,description=}
\newglossaryentry{msods}{name=MSO,description=}

\newglossaryentry{ringtrans}{name=\textsf{TemporalRingTransfer},description=}
\newglossaryentry{memtest}{name=\textsf{MemTest-ST},description=}

\newglossaryentry{copyfirst}{name=\textsc{CopyFirst},description=}
\newglossaryentry{copylast}{name=\textsc{CopyLast},description=}
\newglossaryentry{rocketman}{name=\textsc{RocketMan},description=}

\newglossaryentry{point}{name=Point,description=}
\newglossaryentry{block}{name=Block--T,description=}
\newglossaryentry{block_prop}{name=Block--ST,description=}

\definecolor{drawiopurple}{HTML}{9673A6}
\definecolor{drawiogreen}{HTML}{82B366}
\definecolor{drawiored}{HTML}{B85450}
\definecolor{drawioblue}{HTML}{6C8EBF}
\definecolor{drawioorange}{HTML}{D79B00}
\definecolor{wckdgreen}{HTML}{2A9D8F}
\definecolor{wckdpurple}{HTML}{DF3079}

\usetikzlibrary{arrows.meta, positioning}
\makeatletter
\tikzset{
    vertex/.style={
        circle,
        draw=#1, 
        fill=#1!30, 
        thick,
        minimum size=6pt,
        inner sep=0,
        text centered,
    },
    vertex/.default=black,
    edge/.style={
        thick,
        draw=#1, 
    },
    edge/.default=black,
    diedge/.style={
        thick,
        -latex,
        draw=#1, 
    },
    diedge/.default=black,
}
\makeatother

\usepackage{amsmath, amsfonts, bm}

\def\1{\bm{1}}

\def\rmI{{\mathbf{I}}}

\def\rmN{{\mathbf{N}}}

\def\rmR{{\mathbf{R}}}
\def\rmS{{\mathbf{S}}}

\def\vzero{{\bm{0}}}
\def\vone{{\bm{1}}}

\def\vtheta{{\bm{\theta}}}

\def\ve{{\bm{e}}}

\def\vh{{\bm{h}}}

\def\vx{{\bm{x}}}
\def\vy{{\bm{y}}}
\def\vz{{\bm{z}}}

\def\mA{{\bm{A}}}

\def\mH{{\bm{H}}}
\def\mI{{\bm{I}}}

\def\mT{{\bm{T}}}

\def\mW{{\bm{W}}}
\def\mX{{\bm{X}}}
\def\mY{{\bm{Y}}}
\def\mZ{{\bm{Z}}}

\DeclareMathAlphabet{\mathsfit}{\encodingdefault}{\sfdefault}{m}{sl}
\SetMathAlphabet{\mathsfit}{bold}{\encodingdefault}{\sfdefault}{bx}{n}

\def\gC{{\mathcal{C}}}

\def\gE{{\mathcal{E}}}

\def\gN{{\mathcal{N}}}

\def\gV{{\mathcal{V}}}

\def\sN{{\mathbb{N}}}

\def\sR{{\mathbb{R}}}

\newcommand{\norm}[1]{\left\lVert#1\right\rVert}

\DeclareMathOperator*{\argmin}{arg\,min}
\DeclareMathOperator*{\aggr}{\textsc{Aggr}}

\DeclareMathOperator{\diag}{\mathrm{diag}}

\DeclareMathOperator{\mpl}{\mathsf{MP}}
\DeclareMathOperator{\tmpl}{\mathsf{TMP}}
\DeclareMathOperator{\stmpl}{\mathsf{STMP}}
\DeclareMathOperator{\encl}{\mathsf{ENCODER}}
\DeclareMathOperator{\rol}{\mathsf{READOUT}}
\DeclareMathOperator{\tcnl}{\mathsf{TC}}

\DeclareMathSymbol{\shortminus}{\mathbin}{AMSa}{"39}

\makeatletter
\newcommand{\jac}[2][]{%
  \ifx\relax#1\relax 
    \nabla^{#2}
  \else
    \def\@tempa{#2}%
    \ifx\@tempa\@empty
      \nabla_{#1}
    \else
      \nabla_{#1}^{#2}
    \fi
  \fi%
}
\makeatother

\newcommand{\gso}{\tilde{a}}
\newcommand{\Gso}{\widetilde{\mA}}

\newcommand{\bigo}{\mathcal{O}}

\title{Over-squashing in Spatiotemporal~Graph~Neural~Networks}

\author{%
  Ivan Marisca\textsuperscript{1}\thanks{Work done while at University of Oxford. Correspondance to \texttt{ivan.marisca@usi.ch}}
  \And
  Jacob Bamberger\textsuperscript{2}
  \And
  Cesare Alippi\textsuperscript{1,3}
  \And
  Michael M.\ Bronstein\textsuperscript{2,4}
  \AND\\[-2em]
  \textsuperscript{1}Università della Svizzera italiana, IDSIA, Lugano, Switzerland.\\[.2em]
  \textsuperscript{2}University of Oxford, Oxford, UK.\\[.2em]
  \textsuperscript{3}Politecnico di Milano, Milan, Italy.\\[.2em]
  \textsuperscript{4}AITHYRA, Vienna, Austria.
}

\begin{document}

\maketitle

\begin{abstract}
    Graph Neural Networks (GNNs) have achieved remarkable success across various domains. However, recent theoretical advances have identified fundamental limitations in their information propagation capabilities, such as over-squashing, where distant nodes fail to effectively exchange information. While extensively studied in static contexts, this issue remains unexplored in Spatiotemporal GNNs (STGNNs), which process sequences associated with graph nodes. Nonetheless, the temporal dimension amplifies this challenge by increasing the information that must be propagated. In this work, we formalize the spatiotemporal over-squashing problem and demonstrate its distinct characteristics compared to the static case. Our analysis reveals that, counterintuitively, convolutional STGNNs favor information propagation from points temporally distant rather than close in time. Moreover, we prove that architectures that follow either time-and-space or time-then-space processing paradigms are equally affected by this phenomenon, providing theoretical justification for computationally efficient implementations. We validate our findings on synthetic and real-world datasets, providing deeper insights into their operational dynamics and principled guidance for more effective designs.
\end{abstract}

\section{Introduction}
\label{sec:intro}
Graph deep learning~\citep{battaglia2018relational, bronstein2021geometric} has become a powerful paradigm for learning from relational data, particularly through \glspl{gnn}~\citep{gori2005new, scarselli2008graph, defferrard2016convolutional}. These models process attributed graphs, where nodes represent entities and edges encode their relationships, and have shown strong performance in applications such as drug discovery~\citep{stokes2020deep, liu2023deep}, material synthesis~\citep{merchant2023scaling, szymanski2023autonomous}, and social media analysis~\citep{monti2019fake}.
Beyond static graphs, \glspl{gnn} have also been extended to dynamic settings where graph data evolves over time~\citep{kazemi2020representation, longa2023graph, jin2024survey}. When data can be represented as synchronous sequences associated with nodes of a graph, a common approach is to pair \glspl{gnn} with sequence models like \glspl{rnn}~\citep{hochreiter1997long, cho2014learning} or \glspl{tcn}~\citep{oord2016wavenet, borovykh2017conditional}, leading to the class of \glspl{stgnn}~\citep{seo2018structured, li2018diffusion, cini2023graph}. These architectures have been successfully applied to real-world problems ranging from traffic forecasting~\citep{li2018diffusion, wu2019graph} to energy systems~\citep{khodayar2018spatio, eandi2022spatio} and epidemiology~\citep{kapoor2020examining}.
While the strong empirical performance of \glspl{stgnn} has largely driven their development, much less attention has been devoted to understanding their theoretical capabilities and limitations. This contrasts with the literature on static \glspl{gnn}, which includes extensive analyses on expressivity~\citep{xu2019powerful, bianchi2023expressive} and training dynamics, including over-smoothing~\citep{cai2020note, di_giovanni2023understanding, rusch2023survey} and over-squashing~\citep{alon2021bottleneck, topping2022understanding, di_giovanni2023over}.
In particular, over-squashing -- where information from distant nodes is compressed through graph bottlenecks -- has emerged as a central limitation of message-passing architectures. However, these insights do not directly extend to \glspl{stgnn}, in which information is propagated not only across the \textit{graph} but also along the \textit{time} axis. Understanding how over-squashing occurs in this joint spatiotemporal setting remains an open question.

The temporal axis represents an additional challenge, typically handled by imposing a locality assumption: adjacent time steps are presumed to be more correlated than distant ones. Under this assumption, the sequence behaves like a second, directed graph whose edges encode temporal order. As we stack local filters -- or simply process longer sequences -- increasing volumes of information are stored into fixed-width embeddings, a limitation already noted for recurrent architectures~\citep{bengio1994learning, pascanu2013difficulty}. When this temporal compression meets the constraints of spatial message passing, it creates a compound bottleneck we call \emph{spatiotemporal over-squashing}: messages must cross rapidly expanding receptive fields in both space and time, exceeding the capacity of intermediate representations. \autoref{fig:intrographs} visualizes this dual-graph structure, where information must propagate across both the spatial dimension (blue edges) and temporal dimension (red edges), with each path potentially contributing to the over-squashing phenomenon.

Two dominant architectural strategies have emerged to propagate information both through time and space. \Gls{tts} models first compress each node’s sequence into a vector representation and only afterward propagate these embeddings across the graph, whereas \gls{tas} models interleave temporal and spatial processing so that information flows across both axes at every layer~\citep{gao2022equivalence, cini2023graph}. While the choice between these paradigms has been driven mainly by empirical accuracy and computational costs, we argue that a principled analysis of these trade-offs is essential for guiding the design of future \glspl{stgnn}.

\begin{figure}[t]
    \centering
    \begin{tikzpicture}
    
        \tikzstyle{group}=[draw=black!80, thick, circle, dotted, inner sep=.25mm]
    
        \node[vertex=drawioblue] (v1) at (-.25, .4) {};
        \node[vertex=drawioblue] (v2) at (0, 1) {};
        \node[vertex=drawioblue] (v3) at (.25, .625) {};
        \node[vertex=drawioblue] (v4) at (0, 0) {};
        \node[vertex=drawioblue] (v5) at (0, -.5) {};
        \node[vertex=drawioblue] (v6) at (0, -1) {};
        \node(v1f) at (-.5, .85) {};
        \node(v3f) at (.4, .975) {};

        \node[group, fit=(v5), label={[label distance=0mm]0:$v$}] {};
        
        \draw [edge=drawioblue] (v1)--(v2);
        \draw [edge=drawioblue] (v2)--(v3);
        \draw [edge=drawioblue] (v3)--(v4);
        \draw [edge=drawioblue] (v4)--(v1);
        \draw [edge=drawioblue] (v4)--(v5);
        \draw [edge=drawioblue] (v5)--(v6);
        \draw [edge=drawioblue, densely dashed] (v1)--(v1f);
        \draw [edge=drawioblue, densely dashed] (v3)--(v3f);

    
        \node[anchor=south] at (0, 1.05) {Spatial graph};
        
        \node[right=.5 of v4] (times) {\large$\times$};
        
        \node[right=0 of times] (t0) {};
        \node[vertex=drawiored, right = .7 of t0] (t1) {};
        \node[vertex=drawiored, right = .7 of t1] (t2) {};
        \node[vertex=drawiored, right = .7 of t2] (t3) {};

        \node[group, fit=(t3), label={[label distance=0mm]0:$t$}] {};
        
        \draw [diedge=drawiored, densely dashed] (t0)--(t1);
        \draw [diedge=drawiored] (t1)--(t2);
        \draw [diedge=drawiored] (t2)--(t3);

        \node[fit=(t1)(t2)(t3), label={[label distance=0mm]90:Temporal graph}](t_graph) {};
        
        \node[right=.5 of t3] (equal) {\large$=$};
        
        \def \firstshiftx{5.5cm}
        \node[right=\firstshiftx of v1] (v1_00) {};
        \node[right=\firstshiftx of v2] (v2_00) {};
        \node[right=\firstshiftx of v3] (v3_00) {};
        \node[right=\firstshiftx of v4] (v4_00) {};
        \node[right=\firstshiftx of v5] (v5_00) {};
        \node[right=\firstshiftx of v6] (v6_00) {};
        \node[right=\firstshiftx of v1f] (v1f_00) {};
        \node[right=\firstshiftx of v3f] (v3f_00) {};
        
        \def \gshiftx{.85cm}
        \node[vertex=drawiopurple, right=\gshiftx of v1_00] (v1_1) {};
        \node[vertex=drawiopurple, right=\gshiftx of v2_00] (v2_1) {};
        \node[vertex=drawiopurple, right=\gshiftx of v3_00] (v3_1) {};
        \node[vertex=drawiopurple, right=\gshiftx of v4_00] (v4_1) {};
        \node[vertex=drawiopurple, right=\gshiftx of v5_00] (v5_1) {};
        \node[vertex=drawiopurple, right=\gshiftx of v6_00] (v6_1) {};
        \node[right=\gshiftx of v1f_00] (v1f_1) {};
        \node[right=\gshiftx of v3f_00] (v3f_1) {};

        \draw [diedge=drawiored, densely dashed, opacity=.5] (v1_00)--(v1_1);
        \draw [diedge=drawiored, densely dashed, opacity=.5] (v2_00)--(v2_1);
        \draw [diedge=drawiored, densely dashed, opacity=.5] (v3_00)--(v3_1);
        \draw [diedge=drawiored, densely dashed, opacity=.5] (v4_00)--(v4_1);
        \draw [diedge=drawiored, densely dashed, opacity=.5] (v5_00)--(v5_1);
        \draw [diedge=drawiored, densely dashed, opacity=.5] (v6_00)--(v6_1);

        \draw [edge=drawioblue] (v1_1)--(v2_1);
        \draw [edge=drawioblue] (v2_1)--(v3_1);
        \draw [edge=drawioblue] (v3_1)--(v4_1);
        \draw [edge=drawioblue] (v4_1)--(v1_1);
        \draw [edge=drawioblue] (v4_1)--(v5_1);
        \draw [edge=drawioblue] (v5_1)--(v6_1);
        \draw [edge=drawioblue, densely dashed, opacity=0.75] (v1_1)--(v1f_1);
        \draw [edge=drawioblue, densely dashed, opacity=0.75] (v3_1)--(v3f_1);
        
        \def \gshiftx{.85cm}
        \node[vertex=drawiopurple, right=\gshiftx of v1_1] (v1_2) {};
        \node[vertex=drawiopurple, right=\gshiftx of v2_1] (v2_2) {};
        \node[vertex=drawiopurple, right=\gshiftx of v3_1] (v3_2) {};
        \node[vertex=drawiopurple, right=\gshiftx of v4_1] (v4_2) {};
        \node[vertex=drawiopurple, right=\gshiftx of v5_1] (v5_2) {};
        \node[vertex=drawiopurple, right=\gshiftx of v6_1] (v6_2) {};
        \node[right=\gshiftx of v1f_1] (v1f_2) {};
        \node[right=\gshiftx of v3f_1] (v3f_2) {};
        
        \draw [diedge=drawiored] (v1_1)--(v1_2);
        \draw [diedge=drawiored] (v2_1)--(v2_2);
        \draw [diedge=drawiored] (v3_1)--(v3_2);
        \draw [diedge=drawiored] (v4_1)--(v4_2);
        \draw [diedge=drawiored] (v5_1)--(v5_2);
        \draw [diedge=drawiored] (v6_1)--(v6_2);
        

        \draw [edge=drawioblue] (v1_2)--(v2_2);
        \draw [edge=drawioblue] (v2_2)--(v3_2);
        \draw [edge=drawioblue] (v3_2)--(v4_2);
        \draw [edge=drawioblue] (v4_2)--(v1_2);
        \draw [edge=drawioblue] (v4_2)--(v5_2);
        \draw [edge=drawioblue] (v5_2)--(v6_2);
        \draw [edge=drawioblue, densely dashed, opacity=0.75] (v1_2)--(v1f_2);
        \draw [edge=drawioblue, densely dashed, opacity=0.75] (v3_2)--(v3f_2);
        
        \node[vertex=drawiopurple, right=\gshiftx of v1_2] (v1_3) {};
        \node[vertex=drawiopurple, right=\gshiftx of v2_2] (v2_3) {};
        \node[vertex=drawiopurple, right=\gshiftx of v3_2] (v3_3) {};
        \node[vertex=drawiopurple, right=\gshiftx of v4_2] (v4_3) {};
        \node[vertex=drawiopurple, right=\gshiftx of v5_2] (v5_3) {};
        \node[vertex=drawiopurple, right=\gshiftx of v6_2] (v6_3) {};
        \node[right=\gshiftx of v1f_2] (v1f_3) {};
        \node[right=\gshiftx of v3f_2] (v3f_3) {};

        \node[group, fit=(v5_3), label={[label distance=0mm]0:$(v, t)$}] {};
        
        \draw [diedge=drawiored] (v1_2)--(v1_3);
        \draw [diedge=drawiored] (v2_2)--(v2_3);
        \draw [diedge=drawiored] (v3_2)--(v3_3);
        \draw [diedge=drawiored] (v4_2)--(v4_3);
        \draw [diedge=drawiored] (v5_2)--(v5_3);
        \draw [diedge=drawiored] (v6_2)--(v6_3);
        

        \draw [edge=drawioblue] (v1_3)--(v2_3);
        \draw [edge=drawioblue] (v2_3)--(v3_3);
        \draw [edge=drawioblue] (v3_3)--(v4_3);
        \draw [edge=drawioblue] (v4_3)--(v1_3);
        \draw [edge=drawioblue] (v4_3)--(v5_3);
        \draw [edge=drawioblue] (v5_3)--(v6_3);
        \draw [edge=drawioblue, densely dashed, opacity=0.75] (v1_3)--(v1f_3);
        \draw [edge=drawioblue, densely dashed, opacity=0.75] (v3_3)--(v3f_3);

        \node[anchor=south, above=0 of v2_2] {Spatiotemporal graph};
        
    \end{tikzpicture}
    \caption{Example of spatiotemporal topology governing information propagation in \acrshortpl{stgnn}. The increasing receptive fields of graph-based and sequence-processing architectures compound, as shown in the Cartesian product of spatial and temporal graphs on the right.}
    \label{fig:intrographs}
    \vskip -0.1in
\end{figure}

In this work, we investigate how the interplay between temporal and spatial processing in \glspl{stgnn} shapes learned representations, and how this process is limited by spatiotemporal over-squashing. We analyze information propagation patterns in existing \glspl{stgnn} designs by tracing the sensitivity of each embedding to the input features contained at neighboring nodes and time steps. Specifically, our study focuses on convolutional \glspl{stgnn}, whose temporal component is implemented through a shift operator that exchanges messages between adjacent time steps -- the time-domain analogue of message-passing on graphs~\citep{bai2018empirical}.
In summary, our contributions are the following:

\begin{enumerate}[leftmargin=1.5em, topsep=0pt]
    \item We formally characterize spatiotemporal over-squashing and show its fundamental differences from the static case. The temporal dimension introduces an additional axis for information flow, potentially amplifying the compression effects observed in static graphs. (\autoref{sec:stgnns})
    
    \item We prove both theoretically and empirically that architectures leveraging causal convolutions are, counterintuitively, more sensitive to information far apart in time, and we outline architectural modifications that mitigate this imbalance when required by the task. (\autoref{sec:ovs_tcn})

    \item We demonstrate that spatiotemporal over-squashing affects \acrshort{tas} and \acrshort{tts} paradigms to the same degree. Thus, the computational benefits of \acrshort{tts} models come at no extra cost in terms of information bottlenecks, providing theoretical support for scalable designs. (\autoref{sec:ovs_mptcn})
\end{enumerate}
All theoretical findings are supported by empirical results on both synthetic tasks specifically designed to highlight spatiotemporal bottlenecks, and real-world benchmarks, demonstrating that our insights translate to practical improvements. To our knowledge, no previous work addressed the over-squashing phenomenon in \glspl{stgnn}, despite its potential impact on spatiotemporal modeling.
Our work fills this gap by providing a theoretical framework for understanding information propagation in \glspl{stgnn}, with direct implications for model design and optimization.

\section{Background}
\label{sec:background}

\paragraph{Problem setting}
We denote by $\gV$ the set of $N$ synchronous and regularly-sampled time series, with $\vx_t^v \in \sR^{d_x}$ being the $d_x$-dimensional observation at time step $t$ associated with time series $v \in \gV$. The sequence $\vx^v_{t \shortminus T:t} \in \sR^{T \times d_x}$ indicates the node observations in the interval $(t-T, t]$, $\mX_{t} \in \sR^{N \times d_x}$ the matrix of all observations in $\gV$ at time $t$, so that $\left( \mX_t \right)_v=\vx^v_t$ is the $v$-th entry of $\mX_t$. When referring to a generic node or time step, we omit the indices ${}\cdot{}^v$ and ${}\cdot{}_t$ if not required.
We express temporal dependencies across observations within $\vx_{t \shortminus T:t}$ as a directed path graph $\mT \in \{0,1\}^{T \times T}$, named \textit{temporal graph}, where $(\mT)_{ij}$, i.e., the edge from time step $t - i$ to $t - j$, is $1$ only if $i-j=1$; $\mT$ acts as the \textit{backward shift operator}~\citep{box2015time}.
We assume the existence of relationships across time series, describing dependencies or correlations between the associated observations. We express them as edges in a \textit{spatial graph} with (weighted) adjacency matrix $\mA \in \sR_{\scalemath{.65}{\scalemath{.85}{\geq} 0}}^{N \times N}$, where $a^{uv} = \left( \mA \right)_{uv}$ is nonzero only if there exists an edge connecting node $u$ to $v$.  We use $\Gso$ to indicate a \textit{graph shift operator}, i.e.,  an $N \times N$ real matrix with $\gso^{uv} \neq 0$ if and only if $a^{uv} \neq 0$~\citep{sandryhaila2013discrete}.

We focus on node-level tasks and, given a window of $T$ observations $\mX_{t \shortminus T:t}$ and target label $\mY_{t}$, we consider families of (parametric) models $f_{\vtheta}$ conditioned on the structural dependencies such that
\begin{equation}\label{eq:task}
    \hat{\vy}_{t}^v = \left(f_{\vtheta}\left(\mX_{t \shortminus T:t}, \mA, \mT \right)\right)_v , \quad \forall\ v \in \gV,
\end{equation}
where $\vtheta$ is the set of learnable parameters and $\hat{\vy}_{t}^v$ is the estimate for $\vy_{t}^v$. For classification tasks, $\vy_{t}^v$ encodes the node label, while for prediction, the label is a sequence of $k$ future observations $\vx^v_{t:t+k}$.
Parameters are optimized using a task-dependent loss function, e.g., the \gls{mse}.

\paragraph{Spatiotemporal message passing}
\glspl{stgnn} are architectures specifically designed to process graph-structured data whose node features evolve over discrete time steps~\citep{cini2023graph, jin2024survey}. These models leverage \glspl{gnn} to capture spatial dependencies while employing sequence-processing operators to model temporal dynamics. Among the different \gls{gnn} variants, the primary deep learning approach for relational data are \glspl{mpnn}~\citep{gilmer2017neural}, which operate by iteratively updating each node's representation through aggregation of information from neighboring nodes~\citep{scarselli2008graph, defferrard2016convolutional, bronstein2021geometric}. In an \gls{mpnn}, node representations $\vh^{v(l)} \in \sR^{d}$ at the $l$-th layer are computed through $\mpl^{(l)}$ as
\begin{equation}
    \label{eq:mp}
    \vh^{v(l)} = \left(\mpl^{(l)}\left(\mH^{(l \shortminus 1)}, \Gso \right) \right)_v = \gamma^{(l)}\left(\vh^{v(l \shortminus 1)}, \aggr_{u \in \mathcal{N}(v)} \left\{ \phi^{(l)}\left(\vh^{v(l \shortminus 1)}, \vh^{u(l \shortminus 1)}, \gso^{uv}\right) \right\} \right)
\end{equation}
where $\gamma^{(l)}$ and $\phi^{(l)}$ are differentiable \textit{update} and \textit{message} functions, respectively, $\aggr\{\cdot{}\}$ is a permutation invariant \textit{aggregation} function over the set of messages, and $\mathcal{N}(v)$ is the set of incoming neighbors of $v$.
Borrowing this terminology, in the following we use the term \textit{temporal message-passing}~\citep{marisca2024graph} for any function $\tmpl^{(l)}$ that computes each $i$-th representation $\vh_{t \shortminus i}^{(l)}$ from the sequence $\vh_{t \shortminus T:t}^{(l \shortminus 1)}$ by conditioning on the temporal dependencies defined by $\mT$, i.e.,
\begin{equation}
    \label{eq:tmp}
    \vh^{(l)}_{t \shortminus T:t} = \tmpl^{(l)}\left(\vh_{t \shortminus T:t}^{(l \shortminus 1)}, \mT \right) .
\end{equation}
Examples of this function class include \glspl{rnn} and \glspl{tcn}. While the $\mpl$ and $\tmpl$ operators described previously are constrained to processing along a single dimension, the \textit{spatiotemporal message-passing} layer $\stmpl$ extends this capability to operate simultaneously across both spatial and temporal dimensions~\citep{cini2023taming}. This allows the model to condition its output on both the graph and backward shift-operators $\Gso$ and $\mT$, respectively:
\begin{equation}
    \label{eq:stmpl}
    \vh_{t \shortminus T:t}^{v(l)} = \left(\stmpl^{(l)}\left(\mH_{t \shortminus T : t}^{(l \shortminus 1)}, \Gso, \mT \right)\right)_v .
\end{equation}

\paragraph{Over-squashing in \glspl{gnn}}
This term describes the compression of exponentially increasing information throughout the layers of a \gls{gnn}~\citep{alon2021bottleneck}, which particularly hinders long-range interactions~\citep{bamberger2025measuring}.
A common way to assess over-squashing is by means of a sensitivity analysis through the Jacobian
\begin{equation}
    \label{eq:gnn_jac}
    \jac{u} \vh^{v(L)} := \frac{\partial \vh^{v(L)}}{\partial \vh^{u(0)}} \in \sR^{d \times d} ,
\end{equation}
whose spectral norm $\norm{\jac{u} \vh^{v(L)}}$ acts as a proxy to measure how much initial information at node $u$ can influence the representation computed at node $v$ after $L$ \gls{gnn} processing layers~\citep{topping2022understanding, di_giovanni2023over}. 
For a broad class of \glspl{mpnn}, \citet{di_giovanni2024how} obtained the following bound, which isolates the two contributions of architecture and graph structure to over-squashing.
\begin{restatable}[from \citet{di_giovanni2024how}]{fdgtheorem}{mpnnspec}
    \label{thm:mpnnspec}
    Consider an \gls{mpnn} with $L$ layers, with $c_{\xi}$ being the Lipschitz constant of the update function after activation $\xi$, and $\theta_{\mathsf{m}}$ and $\theta_{\mathsf{u}}$ being the maximal norms over all weight matrices in the message and update functions, respectively.
    For $v, u\in V$ we have:
    \[
        \norm{\jac{u} \vh^{v(L)}} \leq 
        \underbrace{(c_{\xi} \theta_{\mathsf{m}})^L\vphantom{\rmS^L_{\xi}}}_{\scalebox{.7}{$\mathrm{model}$}}\underbrace{\left(\rmS^L\right)_{uv}}_{\scalebox{.7}{$\mathrm{topology}$}} ,
    \]
    where
    \(\smash{
        \rmS  := \frac{\theta_{\mathsf{u}}}{\theta_{\mathsf{m}}}\mI + c_{1}\diag \big(\Gso^{\top} \mathbf{1}\big) + c_{2}\Gso \in \mathbb{R}^{N\times N},
    }\) is the message-passing matrix such that the Jacobian of the message function $\phi^{(l)}$ w.r.t.\ the target ($v$) and neighbor ($u$) node features has bounded norms $c_{1}$ and $c_{2}$, respectively.
\end{restatable}

\section{Information propagation and over-squashing in STGNNs}
\label{sec:stgnns}

In line with previous works~\citep{cini2023taming, cini2023graph}, we consider \glspl{stgnn} obtained by stacking $L$ $\stmpl$ layers~\autorefp{eq:stmpl}, preceded and followed by differentiable encoding and decoding (readout) functions:

\begin{minipage}[b]{.47\linewidth}
    \begin{equation}
    \vh_{t}^{v(0)} = \encl\left(\vx_{t}^v\right), \label{eq:encoder}
    \end{equation}%
\end{minipage}
\hfill
\begin{minipage}[b]{.5\linewidth}
    \begin{equation}
    \hat{\vy}_{t}^{v} = \rol\left(\vh_t^{v (L)}\right) . \label{eq:readout}
    \end{equation}
\end{minipage}\\[.5em]
Note that the encoder is applied independently to each node and time step, while the readout produces estimates for the label using only node representations associated with the last time step. Most existing \glspl{stgnn} can be represented following this framework and differ primarily in the processing carried out in $\stmpl$. In the following, we consider $\stmpl$ operators resulting from the composition of message-passing, i.e., $\mpl$, and sequence-processing, i.e., $\tmpl$, operators.

\subsection{Spatiotemporal message-passing designs}
\label{sec:stmp_design}
A straightforward yet effective strategy to design an $\stmpl$ layer is to factorize processing along the two dimensions with a sequential application of $\mpl$ and $\tmpl$. This enables using existing operators, making the resulting \gls{stgnn} easy to implement. We can write the $l$-th $\stmpl$ layer as:

\begin{minipage}[b]{.47\linewidth}
    \begin{equation}
    \label{eq:tmp_disjoint}
    \vz^{v(l)}_{t \shortminus T:t} = \tmpl_{\scalemath{.67}{\times}L_{\mathsf{T}}}^{(l)}\left( \vh^{v(l \shortminus 1)}_{t \shortminus T:t}, \mT \right) \ \ \forall\, v \in \gV
    \end{equation}%
\end{minipage}
\hfill
\begin{minipage}[b]{.52\linewidth}
    \begin{equation}
    \label{eq:mp_disjoint}
    \vh^{v(l)}_{t \shortminus j} = \left(\mpl_{\scalemath{.67}{\times}L_{\mathsf{S}}}^{(l)}\left(\mZ^{(l)}_{t \shortminus j}, \Gso\right)\right)_v \ \ \forall\, j \in [0, T)
    \end{equation}
\end{minipage}

where $\vz^{v(l)}_{t \shortminus T:t} \in \sR^{T \times d}$ is the sequence of intermediate representations resulting from node-level temporal encoding, and $\mpl$ is applied independently (with shared parameters) across time steps. The subscript ${}\cdot{}_{\scalemath{.67}{\times}L_{\mathsf{T}}}$ (${}\cdot{}_{\scalemath{.67}{\times}L_{\mathsf{S}}}$) concisely denotes a stack of $L_{\mathsf{T}}$ ($L_{\mathsf{S}}$) functions of the same family -- with distinct parameters -- each receiving as input representation the output from the preceding function in the stack. Although processing is decoupled within a single $\stmpl$ layer, the resulting representations effectively incorporate information from the history of neighboring nodes. 

Unlike the graph domain, where \glspl{mpnn} have established themselves as the standard framework for processing relational data, the temporal domain lacks a unified approach that encompasses all architectures. Due to their architectural similarity to \glspl{gnn}, in this work, we focus on \glspl{tcn}, which allows for a more natural extension to spatiotemporal modeling and facilitates drawing analogies between temporal and graph-based representation learning.

\paragraph{\acrshortpl{mptcn}}
We call \glspl{mptcn} those \glspl{stgnn} obtained by combining \glspl{tcn} and \glspl{mpnn} following the framework defined in \autorefseq{eq:tmp_disjoint}{eq:mp_disjoint}~\citep{yu2018spatio, wu2019graph}. Specifically, the $\tmpl$ operator is implemented as a causal convolution of a nonlinear filter with $P$ elements over the temporal dimension~\citep{lecun1995convolutional, bai2018empirical, oord2016wavenet}. Causal convolutions can be expressed through a Toeplitz matrix formulation, which enables the analysis of their propagation dynamics. To formalize this, we introduce the lower-triangular, Toeplitz matrix $\rmR \in \{0,1\}^{T \times T}$ where $(\rmR)_{ij}=1$ if the input at time step $t-i$ influences the output at time step $t-j$. This matrix encodes temporal dependencies analogously to how the graph shift operator $\Gso$ encodes spatial relationships in \glspl{mpnn}. We define the $l$-th layer of a \gls{tcn} $\tcnl^{(l)}$ as:
\begin{equation}
    \label{eq:tcn_ref}
    \vh^{(l)}_{t \shortminus T:t} = \tcnl^{(l)} \left( \vh^{(l \shortminus 1)}_{t \shortminus T:t}, \rmR \right) = \sigma\left( \sum\nolimits_{p=0}^{T} \diag_p(\rmR)^\top\vh^{(l \shortminus 1)}_{t \shortminus T:t} \mW_p^{(l)} \right)
\end{equation}
where each $\mW_{p}^{(l)} \in \sR^{d \times d}$ is a matrix of learnable weights, $\sigma$ is an element-wise activation function (e.g., ReLU), and $\diag_p(\rmR)$ is the matrix obtained by zeroing all entries of $\rmR$ except those on its $p$-th lower diagonal. In standard convolutional implementations, we employ $\rmR = \sum\nolimits_{p=0}^{P - 1} \mT^p$, with $(\rmR)_{ij} = 1$ only if $0 \leq i - j < P$, such that $\diag_p(\rmR) = \mT^p$. This formulation reveals the structural parallels between temporal and spatial message passing, both conditioned on specialized operators ($\rmR$ and $\Gso$, respectively) that encode the underlying topology.

\paragraph{\Acrlong{tas} vs \acrlong{tts}}
\label{sec:tas_vs_tts}
By adjusting the number of outer layers $L$ and inner layers $L_{\mathsf{T}}$ and $L_{\mathsf{S}}$, we can control the degree of temporal-spatial integration while maintaining fixed total processing depth, $LL_{\mathsf{T}}$ and $LL_{\mathsf{S}}$, which we refer to as the \emph{temporal} and \emph{spatial budget} respectively. Fixing both budgets enables a fair experimental comparison between \gls{tts} and \gls{tas} variants. When $L=1$, processing becomes fully decoupled, yielding the computationally efficient \gls{tts} approach that has gained recent prominence~\citep{gao2022equivalence, satorras2022multivariate, cini2023sparse}.
This efficiency arises because encoder-decoder architectures only require representations at time $t$ and layer $L$ for readout (see \autoref{eq:readout}), allowing message passing on a single (static) graph with features $\smash{\mH^{(L)}_t}$.  Thus, with equivalent parameter counts and layer depths, the \gls{tts} approach offers substantially reduced computational complexity for spatial processing, scaling with $\bigo(T)$ -- a detailed discussion of the computational complexities is provided in \autoref{app:compcomp}.
This advantage is particularly valuable in practical applications, where temporal processing can occur asynchronously across nodes before being enriched with spatial context, enabling efficient distributed implementations~\citep{cini2023scalable}. However, \gls{tas} architectures may be more suitable and straightforward to adopt when the graph topology varies over time, with time steps associated with potentially different adjacency matrices.

\subsection{Spatiotemporal over-squashing}
While sensitivity analysis has become a standard approach to studying over-squashing in static \glspl{gnn}~\citep{topping2022understanding, di_giovanni2023over, di_giovanni2024how, gravina2025oversquashing}, it has been limited to graphs with static node features. We extend this analysis to \glspl{stgnn} by examining the sensitivity of node representations after both temporal and spatial processing. In particular, we are interested in studying how information propagation across space and time affects the sensitivity of learned representations to initial node features at previous time steps. Considering that $\smash{\partial \vh_{t}^{v(0)} / \partial \vx_{t \shortminus i}^{u} = \vzero}$ and $\smash{\partial \hat{\vy}_{t}^{v} / \partial \vh_{t \shortminus i}^{u(L)} = \vzero}$ for each $i \neq 0, u\neq v$, we analyze the Jacobian between node features after a stack of $L$ $\stmpl$ layers, i.e.,
\begin{equation}
    \label{eq:stgnn_jac}
    \jac[i]{u} \vh_{t}^{v(L)} := \frac{\partial \vh_{t}^{v(L)}}{\partial \vh_{t \shortminus i}^{u(0)}} \in \sR^{d \times d} .
\end{equation}
This quantity differs conceptually from the simpler static-graph setting~\autorefp{eq:gnn_jac}, as the temporal dimension represents an additional propagation axis. Notably, in decoupled $\stmpl$ functions (\autorefseq{eq:tmp_disjoint}{eq:mp_disjoint}), information flows strictly along separate dimensions: $\tmpl$ operates exclusively within the temporal domain, while $\mpl$ processes only spatial relationships, preventing cross-dimensional interactions. Hence, given an \gls{stgnn} of $L$ layers, for each layer $l \in [1, L]$, nodes $u, v \in \gV$ and $i,j \in [0, T)$ we have
\begin{equation}
    \label{thm:stgnn_fact_jac}
    \frac{\partial \vh_{t \shortminus j}^{v(l)}}{\partial \vh_{t \shortminus i}^{u(l \shortminus 1)}} =  %
    \underbrace{\frac{\partial \vh_{t \shortminus j}^{v(l)}}{\partial \vz_{t \shortminus j}^{u(l)}}}_{\mathrm{space}} %
    \underbrace{\frac{\partial \vz_{t \shortminus j}^{u(l)}}{\partial \vh_{t \shortminus i \vphantom{j}}^{u(l \shortminus 1)}}}_{\mathrm{time}} .
\end{equation}
This factorization allows us to independently study the effects of temporal and spatial processing within each layer on the output representations. Moreover, in the \gls{tts} case where $L=1$, this result provides us a factorized tool to investigate $\jac[i]{u} \vh_{t}^{v(L)}$ and measure how initial representations affect the final output.
While the spatial component has been extensively studied in the literature on \glspl{gnn} (as discussed in \autoref{sec:relwork}), the temporal dimension, especially the phenomenon of \textit{temporal over-squashing}, remains less explored. In the following section, we investigate how temporal processing affects representation learning in \glspl{tcn}, and consequently, \glspl{mptcn}. Following an incremental approach, we first discuss its effects on propagation dynamics in \glspl{tcn} in the next section, and then extend the analysis to the spatiotemporal setting in \autoref{sec:ovs_mptcn}.

\section{Over-squashing in \acrshortpl{tcn}}
\label{sec:ovs_tcn}
To isolate the temporal dynamics, we focus exclusively on the temporal processing component by analyzing encoder-decoder networks constructed from $L_{\mathsf{T}}$ successive $\tcnl$ layers positioned between encoding and readout operations.
In these architectures, each additional layer expands the network's temporal receptive field, enabling the model to capture progressively longer-range dependencies in the input sequence. This hierarchical processing creates an information flow pattern that can be precisely characterized through the powers of the temporal topology matrix $\rmR$. Specifically, $\rmR^l$ represents the temporal receptive field at the $l$-th layer, where entry $(\rmR^l)_{ij}$ quantifies the number of distinct paths through which information can propagate from the input at time step $t-i$ to the representation at time step $t-j$ after $l$ layers of processing.

Understanding how information flows through these paths is crucial for analyzing the model's capacity to effectively leverage temporal context. To formalize this analysis, we investigate how the output representations are influenced by perturbations in the input at preceding time steps. The following theorem establishes a bound on this sensitivity, revealing how the temporal topology governs information flow across layers and drawing important parallels to the graph domain.

\begin{restatable}{theorem}{tcnbound}
    \label{thm:tcnbound}
    Consider a \gls{tcn} with $L_{\mathsf{T}}$ successive $\tcnl$ layers as in \autoref{eq:tcn_ref}, all with kernel size $P$, and assume that $\big\|\mW_{p}^{(l)}\big\| \le \mathsf{w}$ for all $p<P$ and $l\leq L_{\mathsf{T}}$, and that $|\sigma^\prime| \leq c_\sigma$. For each $i, j \in [0, T)$, we have:
    \[
        \norm{\frac{\partial \vh^{(L_{\mathsf{T}})}_{t \shortminus j}}{\partial \vh^{(0)}_{t \shortminus i}}} 
        \leq \underbrace{\left(c_\sigma  \mathsf{w} \right)^{L_{\mathsf{T}}}\vphantom{\left(\rmR^{L_{\mathsf{T}}}\right)_{ij}}}_{\scalebox{.7}{$\mathrm{model}$}} \underbrace{\left(\rmR^{L_{\mathsf{T}}}\right)_{ij}}_{\scalebox{.7}{$\mathrm{temporal\ topology \hspace{-1cm}}$}}.
    \]
\end{restatable}
Proof provided in \appref{app:tempsqu_proof}. Similar to the spatial case in \autoref{thm:mpnnspec}, this bound comprises two components: one dependent on model parameters and another on the temporal topology encoded in $\rmR$. Unlike spatial topologies, however, the temporal structure follows a specific, well-defined pattern that enables deeper theoretical analysis.
The lower-triangular Toeplitz structure of $\rmR$ ensures that its powers maintain this structure \citep{krim2022powers}. This property leads to a distinctive pattern of influence distribution, formalized in the following proposition:

\begin{restatable}{proposition}{tempsqu}
    \label{thm:tempsqu}
    Let \( \rmR \in \sR^{T \times T} \) be a real, lower-triangular, Toeplitz band matrix with lower bandwidth \( P - 1 \), i.e., with \( \left(\rmR\right)_{ij} = r_{i-j} \) for \( 0 \le i-j < P \), and $P\geq 2$, $r_1\neq 0$, and $r_0\neq 0$. Then for any $i>j$ we have $\left|\frac{\left(\rmR^l\right)_{j0}}{\left(\rmR^l\right)_{i0}}\right| \rightarrow 0$ as $l\rightarrow \infty$. In fact $\left|\frac{\left(\rmR^l\right)_{j0}}{\left(\rmR^l\right)_{i0}}\right| = \mathcal{O}(l^{-(i-j)})$. Informally, this means that the final token receives considerably more influence from tokens positioned earlier.
\end{restatable}
%
Proof provided in \appref{app:mptcn_proof}. \autoref{thm:tempsqu} reveals a critical insight: \textbf{causal convolutions progressively diminish sensitivity to recent information while amplifying the influence of temporally distant inputs}. This creates a form of temporal over-squashing that inverts the pattern observed in \glspl{mpnn}, where distant nodes suffer from reduced influence \citep{alon2021bottleneck}. We show this graphically in \autoref{fig:temporal-oversquashing}\hyperref[fig:temporal-oversquashing]{a}, which also displays the powers of the temporal topology matrix (\hyperref[fig:temporal-oversquashing]{a.1}--\hyperref[fig:temporal-oversquashing]{a.2}); a yellow color in matrix entry $(i,j)$ indicates strong influence exerted from time step $t{-}i$ to $t{-}j$. As more convolutional layers are stacked, the influence of temporally recent inputs diminishes relative to more distant ones. This behavior stems from the structure of causal convolutions, which incrementally incorporate more information into a fixed-length context vector. Indeed, causal convolutions propagate information along powers of a directed path graph. Over multiple layers, earlier time steps accumulate influence through an increasing number of propagation paths, while more recent inputs have fewer paths for propagating their initial information. Crucially, because causal convolutions are forward-only, each time step can preserve its information in the associated context vector through self-loops only, with a major impact on the last time step in the sequence.

\begin{figure}[t]
     \begin{subfigure}[t]{.32\linewidth}
        \centering
        \includegraphics[height=2.6cm]{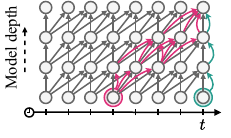}
         \vspace{-.2cm}
         \caption{Standard causal convolution}
         \label{fig:tcn_r_drawing}
     \end{subfigure}
     \begin{subfigure}[t]{.29\linewidth}
        \centering
        \includegraphics[height=2.6cm]{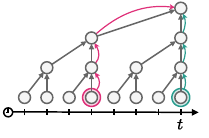}
         \vspace{-.2cm}
         \caption{Dilated convolution}
         \label{fig:tcn_rd_drawing}
     \end{subfigure}
     \begin{subfigure}[t]{.29\linewidth}
        \centering
        \includegraphics[height=2.6cm]{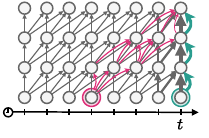}
         \vspace{-.2cm}
         \caption{Normalized convolution}
         \label{fig:tcn_rn_drawing}
         \vspace{.2cm}
     \end{subfigure}
    \begin{subfigure}[t]{\linewidth}
        \includegraphics[width=\linewidth]{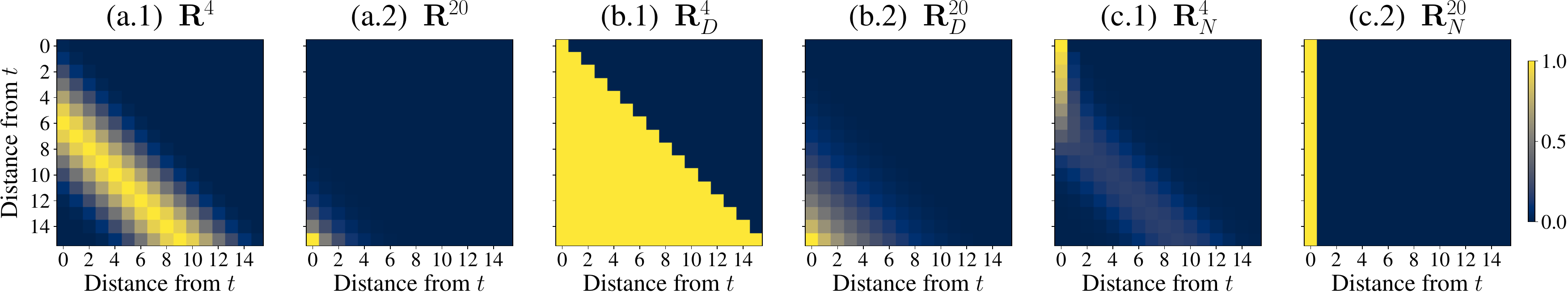}
     \end{subfigure}
    \caption{
    \textbf{Top row:} paths for information flow from the {\color{wckdgreen!78!black} most recent} and an {\color{wckdpurple!86!black} earlier} time step to the last-layer representation at time $t$. 
    \textbf{Bottom row:} evolution of the temporal receptive field after 4 and 20 layers, seen through the powers of the temporal topology matrix. For standard ($\rmR$) and dilated ($\rmR_D$) convolution, the highest-influence region shifts towards the initial time step, while for row-normalized ($\rmR_N$) convolution, we observe a progressive shift to a uniform distribution across all time steps (first column). Entries are scaled matrix-wise in the range $[0,1]$ for comparison purposes.}
    \label{fig:temporal-oversquashing}
    \vspace{-1em}
\end{figure}

When a model's receptive field exceeds the sequence length, i.e., $(P-1)L_{\mathsf{T}} \gg T$, the earliest time step exerts disproportionate influence on the final output compared to any intermediate time step, mirroring the recently investigated \textit{attention sink} effect in Transformers \citep{xiao2023efficient, barbero2024transformers}. This behavior directly undermines the locality bias that causal convolutions are designed to enforce, particularly in time series applications, where recent observations typically carry greater relevance.
To achieve a more balanced receptive field that preserves local information while incorporating broader context, we can act by modifying $\rmR$, effectively implementing \emph{temporal graph rewiring} analogous to techniques used to address over-squashing in \glspl{mpnn} \citep{topping2022understanding, di_giovanni2023over}.

\paragraph{Temporal graph rewiring}
\autoref{thm:tempsqu} outlines that the influence from recent time steps progressively vanishes when the temporal topology (1) remains fixed across layers and (2) maintains a lower-triangular Toeplitz structure. We propose two rewiring approaches targeting these assumptions separately.
Our first approach addresses the fixed topology assumption by employing different matrices $\rmR^{(l)}$ at each layer (each maintaining a lower-triangular Toeplitz structure). This modifies the Jacobian bound in \autoref{thm:tcnbound} to depend on $\prod_{l=1}^{L_{\mathsf{T}}} \rmR^{(l)}$ rather than $\rmR^{L_{\mathsf{T}}}$. Dilated convolutions naturally implement this approach, applying filters with progressively increasing gaps $d^{(l)}$ (dilation rates) between elements~\citep{oord2016wavenet} (\autoref{fig:temporal-oversquashing}\hyperref[fig:temporal-oversquashing]{b}). These convolutions produce matrices $\rmR^{(l)}_{D}$ with nonzero entries $\big(\rmR^{(l)}_{D}\big)_{ij} = r_{i-j}$ only when $i - j = kd^{(l)}$ for $k \in [0, P)$.
When $d^{(l)}=P^{l \shortminus 1}$, the receptive field expands exponentially while ensuring $\big(\prod_{l}\rmR^{(l)}_{D}\big)_{p0} = 1$ for each $p \le P^l$, hence distributing influence equally across all time steps in the receptive field (\autoref{fig:temporal-oversquashing}\hyperref[fig:temporal-oversquashing]{b.1}). Thus, besides efficiency, dilated convolutions have the advantage of preserving local information better than standard convolutions. However, this strategy results in the trivial identity $\rmR_{D}^{(l)}=\mI_T$ for $l > \log_{P} T$, with deeper architectures required to reset the dilation rate every $m$ layers, i.e., $d^{(l)} = P^{(l \shortminus 1) \bmod m}$. While effective in practice, these resets reintroduce over-squashing patterns, as shown in \autoref{fig:temporal-oversquashing}\hyperref[fig:temporal-oversquashing]{b.2}.

Our second approach targets the Toeplitz assumption by row-normalizing $\rmR$ to create $\rmR_{N} = \diag(\rmR\vone)^{-1}\rmR$, where each entry $\left( \rmR_{N} \right)_{ij}$ is normalized by the number of edges from time step $t - i$ (\autoref{fig:temporal-oversquashing}\hyperref[fig:temporal-oversquashing]{c}). This normalization maintains stronger influence from recent time steps while expanding the receptive field, with $\left( \rmR_{N}^l \right)_{i0}$ converging to $1$ for all $i \in [0, T)$ as $l \rightarrow \infty$ (illustrated in \autoref{fig:temporal-oversquashing}\hyperref[fig:temporal-oversquashing]{c.1}--\hyperref[fig:temporal-oversquashing]{c.2} and proven in \autoref{prop:rn_asympt} in the appendix). Despite violating the Toeplitz structure, this approach remains computationally efficient by simply dividing the input at time step $t-i$ by $\min(i+1, P)$. Nonetheless, this normalization primarily benefits the final time step prediction, making it particularly suitable for forecasting tasks where only the last output is used. For tasks requiring readout at intermediate time steps (e.g., imputation), the benefits of such a mitigation may be limited.

\paragraph{Empirical validation} We empirically validate the effects of our proposed temporal convolution modifications through two synthetic sequence memory tasks: \gls{copyfirst} and \gls{copylast}, where the goal is to output, respectively, the first or last observed value in a sequence of $T=16$ random values sampled uniformly in $[0,1]$. Since always predicting $0.5$ yields an \gls{mse} of $\approx 0.083$, we consider a task solved when the test \gls{mse} is lower than $0.001$ and report the success rate across multiple runs. We compare three \glspl{tcn} architectures obtained by stacking $L_{\mathsf{T}}$ $\tcnl$ layers with different temporal convolution topologies: (1) $\rmR$, the standard causal convolution; (2) $\rmR_{N}$, where $\rmR$ is row-normalized; (3) $\rmR_{D}$, implementing dilated convolutions with dilation rates $d^{(l)} = P^{(l \shortminus 1) \bmod m}$ and with $M=4$. We set $P=4$ in $\rmR$ and $\rmR_{N}$, and $P=2$ in $\rmR_{D}$ and vary $L_{\mathsf{T}}$ from $1$ to $20$; \autoref{fig:tcns} shows the simulation results.
For small $L_{\mathsf{T}}$, the \gls{copyfirst} task remains unsolved as the initial time step falls outside the receptive field, while standard convolutions fail on \gls{copylast} when $L_{\mathsf{T}} > 5$ due to the sink phenomenon. As network depth increases, performance degrades across all approaches, despite \gls{copyfirst} remaining more tractable given the sink-induced bias toward earlier time steps, suggesting a fundamental connection to vanishing gradients~\citep{bengio1994learning}.

\begin{figure}[t]
    \includegraphics[width=\textwidth]{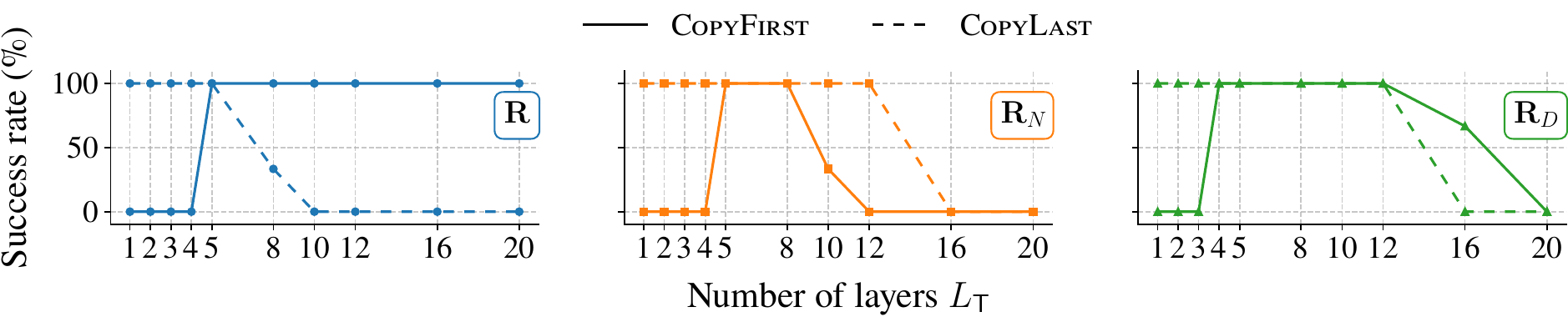}
    \caption{Success rate (\%) on the tasks of copying the first or last observed value across different temporal topologies and number of layers $L_{\mathsf{T}}$.}
    \label{fig:tcns}
    \vspace{-1em}
\end{figure}

\section{Over-squashing in \acrshortpl{mptcn}}
\label{sec:ovs_mptcn}
Having analyzed over-squashing in temporal and spatial domains separately, we now integrate these insights to investigate information flow in \glspl{mptcn}, in which temporal and spatial processing are interleaved across layers. For our analysis, we consider the same class of message-passing functions employed by \citet{di_giovanni2024how} in \autoref{thm:mpnnspec}, which generalizes many popular \glspl{mpnn}~\citep{xu2019powerful, kipf2017semisupervised, atwood2016diffusion, bresson2017residual}. Following the space-time factorization established in \autoref{thm:stgnn_fact_jac}, we observe that within a single layer, the sensitivity is simply the product of the spatial bound from \autoref{thm:mpnnspec} and the temporal bound from \autoref{thm:tcnbound}. For \gls{tts} architectures with $L=1$, this yields:
\begin{equation}
    \norm{\jac[i]{u} \vh_{t}^{v(L)}}  \leq %
        \underbrace{\left(c_{\xi} \theta_{\mathsf{m}}\right)^{L_{\mathsf{S}}} \left(\rmS^{L_{\mathsf{S}}}\right)_{uv}}_{\mathrm{space}} %
        \underbrace{\left(c_\sigma  \mathsf{w} \right)^{L_{\mathsf{T}}}\left(\rmR^{L_{\mathsf{T}}}\right)_{i0}}_{\mathrm{time}} ,
\end{equation}

This bound directly measures how input features influence the final representations used by the readout layer. However, this result only addresses the case where $L=1$ (the \gls{tts} approach). The question remains: how does information propagate when temporal and spatial processing alternate multiple times ($L>1$, the \gls{tas} approach)? To answer this question, we derive the following theorem that characterizes sensitivity across \glspl{mptcn} with any number of alternating processing blocks.

\begin{restatable}{theorem}{mptcn}
\label{thm:mptcn}%
    Consider an \acrshort{mptcn} with $L$ $\stmpl$ layers, each consisting of $L_{\mathsf{T}}$ temporal ($\tmpl$) and $L_{\mathsf{S}}$ spatial ($\mpl$) layers as defined in \autorefseq{eq:tmp_disjoint}{eq:mp_disjoint}. Assume that each $\tmpl$ layer satisfies the conditions of \autoref{thm:tcnbound}, and each $\mpl$ layer satisfies the assumptions in \autoref{thm:mpnnspec}.
    Then, for any $v, u \in \gV$ and $i,j \in [0, T)$, the following holds:
    \[
        \norm{\frac{\partial \vh^{v(L)}_{t \shortminus j}}{\partial \vh^{u(0)}_{t \shortminus i}}} \leq %
        \underbrace{\left(c_{\xi} \theta_{\mathsf{m}}\right)^{LL_{\mathsf{S}}} \left(c_\sigma  \mathsf{w} \right)^{LL_{\mathsf{T}}} \vphantom{\left(\rmR^{LL_{\mathsf{T}}}\right)_{ij}}}_{\mathrm{model}} %
        \underbrace{\left(\rmS^{LL_{\mathsf{S}}}\right)_{uv\vphantom{j}} \left(\rmR^{LL_{\mathsf{T}}}\right)_{ij}}_{\mathrm{spatiotemporal\ topology}}.
    \]
\end{restatable}
The proof is provided in \appref{app:mptcn_proof}. This result bounds the influence of input features $\vh^{u(0)}_{t \shortminus i}$ on output representation $\vh^{v(L)}_{t \shortminus j}$, revealing a clean separation between model parameters and topological factors, as well as between spatial and temporal components. Two key implications emerge from this factorization.
First, the bound's multiplicative structure across space and time dimensions persists regardless of how many $\stmpl$ layers are used. This means that redistributing the computational budget among outer layers $L$ and inner layers $L_{\mathsf{T}}$ and $L_{\mathsf{S}}$ does not alter the bound's characteristics. Therefore, \textbf{from the perspective of information propagation, \gls{tts} architectures ($L=1$) are not inherently limited compared to \gls{tas} architectures ($L>1$)}. While this does not guarantee equivalence in expressivity or optimization dynamics, it provides a principled justification for adopting more computationally efficient \gls{tts} designs without compromising how information flows.
Second, the theorem reveals that \textbf{spatiotemporal over-squashing in \glspl{mptcn} arises from the combined effects of spatial and temporal over-squashing}. This is evident in the spatiotemporal topology component, where both the spatial distance between nodes $u$ and $v$ and the temporal distance between time steps $t-i$ and $t-j$ contribute equally to potential bottlenecks. This insight carries significant practical implications: addressing over-squashing effectively requires targeting both dimensions simultaneously. Improving only one component -- through either spatial or temporal graph rewiring alone -- will prove insufficient if bottlenecks persist in the other dimension.

\begin{figure}[t]
    \vskip 0.1in
    \centering
    \includegraphics[width=\linewidth]{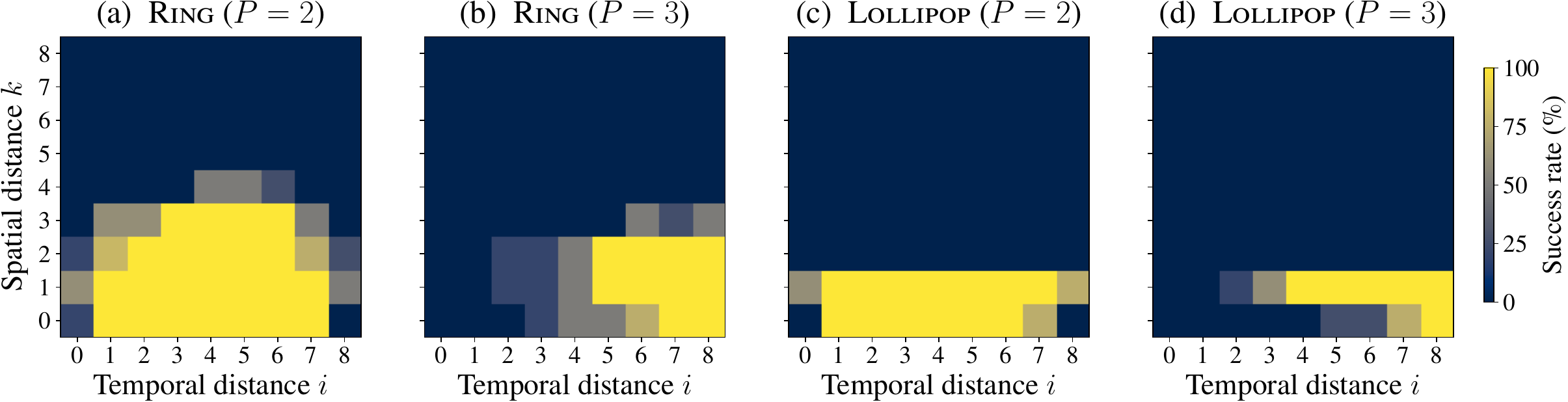}
    \caption{Success rate (\%) of \gls{tts} \acrshort{mptcn}s on the \gls{rocketman} dataset, where the goal is to copy the average value associated with $k$-hop neighbors at time step $t-i$. The tasks vary for the type of graph used (\textsc{Ring} or \textsc{Lollipop}) and size of $P$ ($2$ or $3$). 
    }
    \label{fig:memst}
    \vskip -0.1in
\end{figure}

\subsection{Empirical validation}
To validate our theoretical results, we conducted experiments on both synthetic and real-world tasks that highlight the effects of spatiotemporal over-squashing. The reference \gls{mptcn} architecture used in the experiments features the Diffusion Convolution (DCNN) operator \citep{atwood2016diffusion} as the $\mpl$ layer.

\paragraph{Synthetic environments}
We design a synthetic memory task named \gls{rocketman} where, given a graph and a time window of random values, the model must retrieve the average value at time step $t-i$ for nodes exactly $k$ hops away from a target node. We keep the input size constant across tasks and employ a \gls{tts} architecture with precisely enough layers in each dimension to span the entire input space, reporting success rates across multiple runs.
The results, shown in \autoref{fig:memst}, reveal two key patterns consistent with our theory. First, the task is significantly more challenging on the \textsc{Lollipop} graph compared to the \textsc{Ring} graph, confirming the known spatial over-squashing characteristics of these topologies~\citep{di_giovanni2024how}. Second, as the convolutional filter size $P$ increases, we observe that the task becomes more challenging for lower temporal distances, aligning perfectly with our analysis of \glspl{tcn}. Complete experimental details, including comparative results for a \gls{tas} architecture, are provided in \appref{app:experiments}, while in \appref{app:neigh_match} we show results on \textsc{TemporalNeighboursMatch}, an adaptation of the synthetic environment \textsc{NeighboursMatch} \citep{alon2021bottleneck} to the spatiotemporal setting.

\paragraph{Real-world benchmarks}
\begin{wraptable}{R}{.5\textwidth}
\vspace{-.4cm}
\centering
\small
\caption{Forecasting error (MAE) of \acrshortpl{mptcn} with fixed budget in real-world benchmarks.}
\label{tab:realworld}
\setlength{\tabcolsep}{3pt}
\setlength{\aboverulesep}{0pt}
\setlength{\belowrulesep}{0pt}
\renewcommand{\arraystretch}{1.35}
\resizebox{.49\textwidth}{!}{%
\begin{tabular}{c|c|c|ccc}
        \toprule
        \multicolumn{2}{c|}{\sc Models} & $L$ & \gls{la} & \gls{bay} & \gls{engrad} \\
        \midrule
        \multirow{6}{*}{\rotatebox[origin=c]{90}{\acrshort{mptcn}}} & \multirow{3}{*}{$\rmR$} & 6 & 3.19{{\tiny $\pm$0.02}} & 1.66{{\tiny $\pm$0.00}} & 44.43{{\tiny $\pm$0.41}} \\
        & & 3 & 3.19{{\tiny $\pm$0.01}} & 1.65{{\tiny $\pm$0.01}} & 43.83{{\tiny $\pm$0.03}} \\
        & & 1 & 3.14{{\tiny $\pm$0.02}} & 1.63{{\tiny $\pm$0.01}} & 44.47{{\tiny $\pm$0.42}} \\
        \cmidrule[.4pt]{2-6}
        & \multirow{3}{*}{$\rmR_{N}$} & 6 & 3.17{{\tiny $\pm$0.02}} & 1.65{{\tiny $\pm$0.01}} & 41.82{{\tiny $\pm$0.38}} \\
        & & 3 & 3.17{{\tiny $\pm$0.01}} & 1.65{{\tiny $\pm$0.00}} & 41.78{{\tiny $\pm$0.09}} \\
        & & 1 & 3.16{{\tiny $\pm$0.01}} & 1.65{{\tiny $\pm$0.01}} & 40.38{{\tiny $\pm$0.08}} \\
        \bottomrule
        \bottomrule
        \multicolumn{3}{c|}{\acrshort{gwnet} (orig.)} & 3.02{{\tiny $\pm$0.02}} & 1.55{{\tiny $\pm$0.01}} & 40.50{{\tiny $\pm$0.27}} \\
        \cmidrule[.4pt]{1-6}
        \multicolumn{3}{c|}{\acrshort{gwnet} \acrshort{tts}} & 3.00{{\tiny $\pm$0.01}} & 1.57{{\tiny $\pm$0.00}} & 40.64{{\tiny $\pm$0.29}} \\
        \bottomrule
\end{tabular}
}
\vspace{-4pt}
\end{wraptable}

To verify that our theoretical insights extend to practical applications, we evaluated various \glspl{mptcn} of \emph{fixed spatial and temporal budgets} $LL_{\mathsf{S}}$ and $LL_{\mathsf{T}}$ respectively, on three spatiotemporal forecasting benchmarks: \gls{la}~\citep{li2018diffusion}, \gls{bay}~\citep{li2018diffusion}, and \gls{engrad}~\citep{marisca2024graph}. \autoref{tab:realworld} presents the prediction errors in terms of \gls{mae}.
The results offer several insights that support our theoretical analysis. First, \gls{tas} and \gls{tts} approaches perform comparably on average, with the more efficient \gls{tts} models outperforming in the majority of cases. This empirical finding supports our theoretical conclusion that the \gls{tts} design is not inherently limited from an information propagation perspective, although additional factors beyond may also influence comparative performance.
Second, in the \gls{engrad} dataset, which requires larger filter sizes to cover the input sequence, the row-normalization approach consistently improves performance. This hints that temporal over-squashing affects standard convolutional models in practical scenarios where longer temporal contexts are needed, as our theory suggests.
Finally, we compare \gls{gwnet}~\citep{wu2019graph} -- a widely-used and more complex \gls{tas} architecture -- against its \gls{tts} counterpart. Results show that our findings remain valid even when more sophisticated architectural components are involved.

\paragraph{Combining spatial and temporal rewiring}
\begin{wraptable}{R}{.5\textwidth}
\centering
\small
\caption{Forecasting error (MAE) of \acrshort{tts} \acrshortpl{mptcn} with and without spatial and temporal graph rewiring in \gls{engrad}. 
}
\label{tab:spatiotemporal_rewiring}
\setlength{\tabcolsep}{6pt}
\setlength{\aboverulesep}{0pt}
\setlength{\belowrulesep}{0pt}
\renewcommand{\arraystretch}{1.5}
\vspace{-3pt}
\begin{tabular}{c|c|ccc}
        \cmidrule[.8pt]{2-4}
         \multicolumn{1}{c}{} & \multirow{2}{*}{\parbox{1cm}{\centering Original graph}} & \multicolumn{2}{c}{\textsc{FoSR} rewiring} \\
        \cmidrule[.5pt]{3-4}
         \multicolumn{1}{c}{} & & w/ RGCN & w/ DCNN\\
        \toprule
        {$\rmR$} & \cellcolor{gray!25}{44.47{\tiny $\pm$0.42}} & \cellcolor{gray!20}{43.78{\tiny $\pm$0.29}} & \cellcolor{gray!15}{43.50{\tiny $\pm$0.08}} \\
        \midrule
        {$\rmR_N$} & \cellcolor{gray!5}{40.38{\tiny $\pm$0.08}} & \cellcolor{gray!10}{41.10{\tiny $\pm$0.11}} & 40.30{\tiny $\pm$0.16} \\
        \bottomrule
\end{tabular}
\vspace{-5pt}
\end{wraptable}

In this experiment, we assess the combined effect of spatial and temporal graph rewiring in alleviating spatiotemporal over-squashing. We adopt \textsc{FoSR} \citep{karhadkar2023fosr} as the graph rewiring method and evaluate the forecasting error with and without row-normalized convolutions. We test the models on \gls{engrad} for its symmetric topology, which makes rewiring more meaningful compared to traffic forecasting tasks, where spatial structure is rigidly defined by the underlying road network. Besides our original \gls{mptcn} implementation relying on DCNN as $\mpl$ layer, we further consider RGCN \citep{schlichtkrull2018modeling}, to weight differently the contribution of rewired edges. We report results in the \gls{tts} setting in \autoref{tab:spatiotemporal_rewiring}. We can observe that combining both spatial and temporal rewiring yields the best performance in the original implementation using DCNN. In particular, rewiring in each dimension individually improves accuracy, with the temporal one contributing the largest marginal gain. This is consistent with our theoretical analysis, reinforcing that temporal bottlenecks are a significant limiting factor in \glspl{stgnn}.

\section{Related work}
\label{sec:relwork}

The issue of over-squashing in \glspl{gnn} was first highlighted by \citet{alon2021bottleneck}, who showed that \glspl{gnn} struggle to capture long-range dependencies in graphs with structural bottlenecks. Building on this, \citet{topping2022understanding} introduced a sensitivity-based framework to identify over-squashing, which was later extended by \citet{di_giovanni2023over, di_giovanni2024how}. Two main strategies have emerged to alleviate this problem: graph re-wiring to enhance connectivity \cite{gasteiger2019diffusion, topping2022understanding, karhadkar2023fosr, barbero2024locality}, and architectural modifications to stabilize gradients \cite{chen2020measuring, tortorella2022leave, gravina2023anti, gravina2025oversquashing}. Notably, these efforts largely overlook the role of time-evolving node features.
Similar challenges in modeling long-range dependencies have been studied in temporal architectures like \glspl{rnn}, where vanishing and exploding gradients hinder effective learning \cite{bengio1994learning, pascanu2013difficulty}. Solutions include enforcing orthogonality \cite{vorontsov2017orthogonality} or antisymmetry \cite{chang2019antisymmetricrnn} in weight matrices, as well as designing specialized stable architectures~\citep{rusch2021unicornn}. Recently, the attention sink effect~\cite{xiao2023efficient} has revealed a bias in Transformers towards early tokens~\citep{wu2025emergence}, with \citet{barbero2024transformers} linking this phenomenon to over-squashing in language models.
In the spatiotemporal domain, \citet{gao2022equivalence} analyzed the expressive power of \glspl{stgnn}, showing their capabilities and limits in distinguishing non-isomorphic graphs in both \gls{tas} and \gls{tts} settings, while \citet{gravina2024long} propose a framework tailored to long-range tasks in continuous-time dynamic graphs. Yet, to our knowledge, no prior work has directly tackled the problem of spatiotemporal over-squashing in \glspl{stgnn}, where the interaction between spatial and temporal dimensions introduces unique challenges for information propagation.

\section{Conclusions}
\label{sec:conclusions}

In this work, we have formally characterized spatiotemporal over-squashing in \glspl{stgnn}, demonstrating its distinctions from the static case. Our analysis reveals that convolutional \glspl{stgnn} counterintuitively favor information propagation from temporally distant points, offering key insights into their behavior. Despite their structural differences, we proved that both \gls{tas} and \gls{tts} paradigms equally suffer from this phenomenon, providing theoretical justification for computationally efficient implementations. Experiments on synthetic and real-world datasets confirm our theoretical framework's practical relevance.
The insights from this study directly impact the design of effective spatiotemporal architectures. By bridging theory and practice, we contribute to a deeper understanding of \glspl{stgnn} and provide principled guidance for their implementation.
\vspace{-.73em}
\paragraph{Limitations and future work}
We focus on factorized convolutional \glspl{stgnn}, aligning with established \gls{gnn} research while extending existing theoretical results to the spatiotemporal domain. This choice allows for clean theoretical bounds and a controlled setting applicable to both \gls{tts} and \gls{tas} variants. However, it also limits the generality of our results to models where cross-dimensional interactions are blocked. Extending our framework to models with joint space–time filters or recurrent \glspl{stgnn}, which follow fundamentally different propagation dynamics, represents a valuable direction for future work.
Finally, our sensitivity bounds are derived in the worst case and are therefore potentially conservative; deriving tighter, data-dependent estimates and conducting a systematic analysis of mitigation strategies remain promising directions for future research.

\begin{ack}
The authors wish to thank Francesco Di Giovanni for the valuable feedback and collaboration during the initial phase of this research. M.B.\ and J.B.\ are partially supported by the EPSRC Turing AI World-Leading Research Fellowship No. EP/X040062/1 and EPSRC AI Hub No. EP/Y028872/1. C.A.\ is partly supported by the Swiss National Science Foundation under grant no. 204061 \emph{HORD GNN: Higher-Order Relations and Dynamics in Graph Neural Networks} and the International Partnership Program of the Chinese Academy of Sciences under Grant 104GJHZ2022013GC.
\end{ack}

\bibliography{bibliography}

\begin{thebibliography}{78}
\providecommand{\natexlab}[1]{#1}
\providecommand{\url}[1]{\texttt{#1}}
\expandafter\ifx\csname urlstyle\endcsname\relax
  \providecommand{\doi}[1]{doi: #1}\else
  \providecommand{\doi}{doi: \begingroup \urlstyle{rm}\Url}\fi

\bibitem[Battaglia et~al.(2018)Battaglia, Hamrick, Bapst, Sanchez-Gonzalez, Zambaldi, Malinowski, Tacchetti, Raposo, Santoro, Faulkner, et~al.]{battaglia2018relational}
Peter~W Battaglia, Jessica~B Hamrick, Victor Bapst, Alvaro Sanchez-Gonzalez, Vinicius Zambaldi, Mateusz Malinowski, Andrea Tacchetti, David Raposo, Adam Santoro, Ryan Faulkner, et~al.
\newblock Relational inductive biases, deep learning, and graph networks.
\newblock \emph{arXiv preprint arXiv:1806.01261}, 2018.

\bibitem[Bronstein et~al.(2021)Bronstein, Bruna, Cohen, and Veli{\v{c}}kovi{\'c}]{bronstein2021geometric}
Michael~M Bronstein, Joan Bruna, Taco Cohen, and Petar Veli{\v{c}}kovi{\'c}.
\newblock Geometric deep learning: Grids, groups, graphs, geodesics, and gauges.
\newblock \emph{arXiv preprint arXiv:2104.13478}, 2021.

\bibitem[Gori et~al.(2005)Gori, Monfardini, and Scarselli]{gori2005new}
M.~Gori, G.~Monfardini, and F.~Scarselli.
\newblock A new model for learning in graph domains.
\newblock In \emph{Proceedings. 2005 IEEE International Joint Conference on Neural Networks}, volume~2, pages 729--734, 2005.

\bibitem[Scarselli et~al.(2008)Scarselli, Gori, Tsoi, Hagenbuchner, and Monfardini]{scarselli2008graph}
Franco Scarselli, Marco Gori, Ah~Chung Tsoi, Markus Hagenbuchner, and Gabriele Monfardini.
\newblock The graph neural network model.
\newblock \emph{IEEE transactions on neural networks}, 20\penalty0 (1):\penalty0 61--80, 2008.

\bibitem[Defferrard et~al.(2016)Defferrard, Bresson, and Vandergheynst]{defferrard2016convolutional}
Micha{\"e}l Defferrard, Xavier Bresson, and Pierre Vandergheynst.
\newblock Convolutional neural networks on graphs with fast localized spectral filtering.
\newblock \emph{Advances in neural information processing systems}, 29, 2016.

\bibitem[Stokes et~al.(2020)Stokes, Yang, Swanson, Jin, Cubillos-Ruiz, Donghia, MacNair, French, Carfrae, Bloom-Ackermann, et~al.]{stokes2020deep}
Jonathan~M Stokes, Kevin Yang, Kyle Swanson, Wengong Jin, Andres Cubillos-Ruiz, Nina~M Donghia, Craig~R MacNair, Shawn French, Lindsey~A Carfrae, Zohar Bloom-Ackermann, et~al.
\newblock A deep learning approach to antibiotic discovery.
\newblock \emph{Cell}, 180\penalty0 (4):\penalty0 688--702, 2020.

\bibitem[Liu et~al.(2023)Liu, Catacutan, Rathod, Swanson, Jin, Mohammed, Chiappino-Pepe, Syed, Fragis, Rachwalski, et~al.]{liu2023deep}
Gary Liu, Denise~B Catacutan, Khushi Rathod, Kyle Swanson, Wengong Jin, Jody~C Mohammed, Anush Chiappino-Pepe, Saad~A Syed, Meghan Fragis, Kenneth Rachwalski, et~al.
\newblock Deep learning-guided discovery of an antibiotic targeting acinetobacter baumannii.
\newblock \emph{Nature Chemical Biology}, 19\penalty0 (11):\penalty0 1342--1350, 2023.

\bibitem[Merchant et~al.(2023)Merchant, Batzner, Schoenholz, Aykol, Cheon, and Cubuk]{merchant2023scaling}
Amil Merchant, Simon Batzner, Samuel~S Schoenholz, Muratahan Aykol, Gowoon Cheon, and Ekin~Dogus Cubuk.
\newblock Scaling deep learning for materials discovery.
\newblock \emph{Nature}, 624\penalty0 (7990):\penalty0 80--85, 2023.

\bibitem[Szymanski et~al.(2023)Szymanski, Rendy, Fei, Kumar, He, Milsted, McDermott, Gallant, Cubuk, Merchant, et~al.]{szymanski2023autonomous}
Nathan~J Szymanski, Bernardus Rendy, Yuxing Fei, Rishi~E Kumar, Tanjin He, David Milsted, Matthew~J McDermott, Max Gallant, Ekin~Dogus Cubuk, Amil Merchant, et~al.
\newblock An autonomous laboratory for the accelerated synthesis of novel materials.
\newblock \emph{Nature}, 624\penalty0 (7990):\penalty0 86--91, 2023.

\bibitem[Monti et~al.(2019)Monti, Frasca, Eynard, Mannion, and Bronstein]{monti2019fake}
Federico Monti, Fabrizio Frasca, Davide Eynard, Damon Mannion, and Michael~M Bronstein.
\newblock Fake news detection on social media using geometric deep learning.
\newblock \emph{arXiv preprint arXiv:1902.06673}, 2019.

\bibitem[Kazemi et~al.(2020)Kazemi, Goel, Jain, Kobyzev, Sethi, Forsyth, and Poupart]{kazemi2020representation}
Seyed~Mehran Kazemi, Rishab Goel, Kshitij Jain, Ivan Kobyzev, Akshay Sethi, Peter Forsyth, and Pascal Poupart.
\newblock Representation learning for dynamic graphs: A survey.
\newblock \emph{Journal of Machine Learning Research}, 21\penalty0 (70):\penalty0 1--73, 2020.

\bibitem[Longa et~al.(2023)Longa, Lachi, Santin, Bianchini, Lepri, Lio, Scarselli, and Passerini]{longa2023graph}
Antonio Longa, Veronica Lachi, Gabriele Santin, Monica Bianchini, Bruno Lepri, Pietro Lio, Franco Scarselli, and Andrea Passerini.
\newblock Graph neural networks for temporal graphs: State of the art, open challenges, and opportunities.
\newblock \emph{Transactions on Machine Learning Research}, 2023.
\newblock ISSN 2835-8856.

\bibitem[Jin et~al.(2024)Jin, Koh, Wen, Zambon, Alippi, Webb, King, and Pan]{jin2024survey}
Ming Jin, Huan~Yee Koh, Qingsong Wen, Daniele Zambon, Cesare Alippi, Geoffrey~I Webb, Irwin King, and Shirui Pan.
\newblock A survey on graph neural networks for time series: Forecasting, classification, imputation, and anomaly detection.
\newblock \emph{IEEE Transactions on Pattern Analysis and Machine Intelligence}, 2024.

\bibitem[Hochreiter and Schmidhuber(1997)]{hochreiter1997long}
Sepp Hochreiter and J{\"u}rgen Schmidhuber.
\newblock Long short-term memory.
\newblock \emph{Neural computation}, 9\penalty0 (8):\penalty0 1735--1780, 1997.

\bibitem[Cho et~al.(2014)Cho, Van~Merri{\"e}nboer, Gulcehre, Bahdanau, Bougares, Schwenk, and Bengio]{cho2014learning}
Kyunghyun Cho, Bart Van~Merri{\"e}nboer, Caglar Gulcehre, Dzmitry Bahdanau, Fethi Bougares, Holger Schwenk, and Yoshua Bengio.
\newblock Learning phrase representations using rnn encoder--decoder for statistical machine translation.
\newblock In \emph{Proceedings of the 2014 Conference on Empirical Methods in Natural Language Processing (EMNLP)}, pages 1724--1734, 2014.

\bibitem[Oord et~al.(2016)Oord, Dieleman, Zen, Simonyan, Vinyals, Graves, Kalchbrenner, Senior, and Kavukcuoglu]{oord2016wavenet}
Aaron van~den Oord, Sander Dieleman, Heiga Zen, Karen Simonyan, Oriol Vinyals, Alex Graves, Nal Kalchbrenner, Andrew Senior, and Koray Kavukcuoglu.
\newblock Wavenet: A generative model for raw audio.
\newblock \emph{arXiv preprint arXiv:1609.03499}, 2016.

\bibitem[Borovykh et~al.(2017)Borovykh, Bohte, and Oosterlee]{borovykh2017conditional}
Anastasia Borovykh, Sander Bohte, and Cornelis~W Oosterlee.
\newblock Conditional time series forecasting with convolutional neural networks.
\newblock \emph{arXiv preprint arXiv:1703.04691}, 2017.

\bibitem[Seo et~al.(2018)Seo, Defferrard, Vandergheynst, and Bresson]{seo2018structured}
Youngjoo Seo, Micha{\"e}l Defferrard, Pierre Vandergheynst, and Xavier Bresson.
\newblock Structured sequence modeling with graph convolutional recurrent networks.
\newblock In \emph{International Conference on Neural Information Processing}, pages 362--373. Springer, 2018.

\bibitem[Li et~al.(2018)Li, Yu, Shahabi, and Liu]{li2018diffusion}
Yaguang Li, Rose Yu, Cyrus Shahabi, and Yan Liu.
\newblock Diffusion convolutional recurrent neural network: Data-driven traffic forecasting.
\newblock In \emph{International Conference on Learning Representations}, 2018.

\bibitem[Cini et~al.(2025)Cini, Marisca, Zambon, and Alippi]{cini2023graph}
Andrea Cini, Ivan Marisca, Daniele Zambon, and Cesare Alippi.
\newblock Graph deep learning for time series forecasting.
\newblock \emph{ACM Computing Surveys}, 57\penalty0 (12), June 2025.
\newblock ISSN 0360-0300.
\newblock \doi{10.1145/3742784}.

\bibitem[Wu et~al.(2019)Wu, Pan, Long, Jiang, and Zhang]{wu2019graph}
Zonghan Wu, Shirui Pan, Guodong Long, Jing Jiang, and Chengqi Zhang.
\newblock Graph wavenet for deep spatial-temporal graph modeling.
\newblock In \emph{Proceedings of the {{Twenty-Eighth International Joint Conference}} on {{Artificial Intelligence}}}, pages 1907--1913. Association for the Advancement of Artificial Intelligence (AAAI), 2019.
\newblock \doi{10.24963/ijcai.2019/264}.

\bibitem[Khodayar and Wang(2018)]{khodayar2018spatio}
Mahdi Khodayar and Jianhui Wang.
\newblock Spatio-temporal graph deep neural network for short-term wind speed forecasting.
\newblock \emph{IEEE Transactions on Sustainable Energy}, 10\penalty0 (2):\penalty0 670--681, 2018.

\bibitem[Eandi et~al.(2022)Eandi, Cini, Lukovic, and Alippi]{eandi2022spatio}
Simone Eandi, Andrea Cini, Slobodan Lukovic, and Cesare Alippi.
\newblock Spatio-temporal graph neural networks for aggregate load forecasting.
\newblock In \emph{2022 International Joint Conference on Neural Networks (IJCNN)}, pages 1--8. IEEE, 2022.

\bibitem[Kapoor et~al.(2020)Kapoor, Ben, Liu, Perozzi, Barnes, Blais, and O'Banion]{kapoor2020examining}
Amol Kapoor, Xue Ben, Luyang Liu, Bryan Perozzi, Matt Barnes, Martin Blais, and Shawn O'Banion.
\newblock Examining covid-19 forecasting using spatio-temporal gnns.
\newblock In \emph{Proceedings of the 16th International Workshop on Mining and Learning with Graphs (MLG)}, 2020.

\bibitem[Xu et~al.(2019)Xu, Hu, Leskovec, and Jegelka]{xu2019powerful}
Keyulu Xu, Weihua Hu, Jure Leskovec, and Stefanie Jegelka.
\newblock How powerful are graph neural networks?
\newblock In \emph{International {{Conference}} on {{Learning Representations}}}, 2019.

\bibitem[Bianchi and Lachi(2023)]{bianchi2023expressive}
Filippo~Maria Bianchi and Veronica Lachi.
\newblock The expressive power of pooling in graph neural networks.
\newblock In \emph{Advances in Neural Information Processing Systems}, volume~36, pages 71603--71618. Curran Associates, Inc., 2023.

\bibitem[Cai and Wang(2020)]{cai2020note}
Chen Cai and Yusu Wang.
\newblock A note on over-smoothing for graph neural networks.
\newblock \emph{arXiv preprint arXiv:2006.13318}, 2020.

\bibitem[Giovanni et~al.(2023)Giovanni, Rowbottom, Chamberlain, Markovich, and Bronstein]{di_giovanni2023understanding}
Francesco~Di Giovanni, James Rowbottom, Benjamin~Paul Chamberlain, Thomas Markovich, and Michael~M. Bronstein.
\newblock Understanding convolution on graphs via energies.
\newblock \emph{Transactions on Machine Learning Research}, 2023.
\newblock ISSN 2835-8856.

\bibitem[Rusch et~al.(2023)Rusch, Bronstein, and Mishra]{rusch2023survey}
T~Konstantin Rusch, Michael~M Bronstein, and Siddhartha Mishra.
\newblock A survey on oversmoothing in graph neural networks.
\newblock \emph{arXiv preprint arXiv:2303.10993}, 2023.

\bibitem[Alon and Yahav(2021)]{alon2021bottleneck}
Uri Alon and Eran Yahav.
\newblock On the bottleneck of graph neural networks and its practical implications.
\newblock In \emph{International Conference on Learning Representations}, 2021.

\bibitem[Topping et~al.(2022)Topping, Di~Giovanni, Chamberlain, Dong, and Bronstein]{topping2022understanding}
Jake Topping, Francesco Di~Giovanni, Benjamin~Paul Chamberlain, Xiaowen Dong, and Michael~M Bronstein.
\newblock Understanding over-squashing and bottlenecks on graphs via curvature.
\newblock In \emph{International Conference on Learning Representations}, 2022.

\bibitem[Di~Giovanni et~al.(2023)Di~Giovanni, Giusti, Barbero, Luise, Lio, and Bronstein]{di_giovanni2023over}
Francesco Di~Giovanni, Lorenzo Giusti, Federico Barbero, Giulia Luise, Pietro Lio, and Michael~M Bronstein.
\newblock On over-squashing in message passing neural networks: The impact of width, depth, and topology.
\newblock In \emph{International Conference on Machine Learning}, pages 7865--7885. PMLR, 2023.

\bibitem[Bengio et~al.(1994)Bengio, Simard, and Frasconi]{bengio1994learning}
Yoshua Bengio, Patrice Simard, and Paolo Frasconi.
\newblock Learning long-term dependencies with gradient descent is difficult.
\newblock \emph{IEEE transactions on neural networks}, 5\penalty0 (2):\penalty0 157--166, 1994.

\bibitem[Pascanu et~al.(2013)Pascanu, Mikolov, and Bengio]{pascanu2013difficulty}
Razvan Pascanu, Tomas Mikolov, and Yoshua Bengio.
\newblock On the difficulty of training recurrent neural networks.
\newblock In \emph{Proceedings of the 30th International Conference on Machine Learning}, volume~28 of \emph{Proceedings of Machine Learning Research}, pages 1310--1318, Atlanta, Georgia, USA, June 2013. PMLR.
\newblock URL \url{https://proceedings.mlr.press/v28/pascanu13.html}.

\bibitem[Gao and Ribeiro(2022)]{gao2022equivalence}
Jianfei Gao and Bruno Ribeiro.
\newblock On the equivalence between temporal and static equivariant graph representations.
\newblock In \emph{International Conference on Machine Learning}, pages 7052--7076. PMLR, 2022.

\bibitem[Bai et~al.(2018)Bai, Kolter, and Koltun]{bai2018empirical}
Shaojie Bai, J.~Zico Kolter, and Vladlen Koltun.
\newblock An empirical evaluation of generic convolutional and recurrent networks for sequence modeling.
\newblock \emph{arXiv:1803.01271}, 2018.

\bibitem[Box et~al.(2015)Box, Jenkins, Reinsel, and Ljung]{box2015time}
George~EP Box, Gwilym~M Jenkins, Gregory~C Reinsel, and Greta~M Ljung.
\newblock \emph{Time series analysis: forecasting and control}.
\newblock John Wiley \& Sons, 2015.

\bibitem[Sandryhaila and Moura(2013)]{sandryhaila2013discrete}
Aliaksei Sandryhaila and Jos{\'e}~MF Moura.
\newblock Discrete signal processing on graphs.
\newblock \emph{IEEE transactions on signal processing}, 61\penalty0 (7):\penalty0 1644--1656, 2013.

\bibitem[Gilmer et~al.(2017)Gilmer, Schoenholz, Riley, Vinyals, and Dahl]{gilmer2017neural}
Justin Gilmer, Samuel~S Schoenholz, Patrick~F Riley, Oriol Vinyals, and George~E Dahl.
\newblock Neural message passing for quantum chemistry.
\newblock In \emph{International conference on machine learning}, pages 1263--1272. PMLR, 2017.

\bibitem[Marisca et~al.(2024)Marisca, Alippi, and Bianchi]{marisca2024graph}
Ivan Marisca, Cesare Alippi, and Filippo~Maria Bianchi.
\newblock Graph-based forecasting with missing data through spatiotemporal downsampling.
\newblock In \emph{Proceedings of the 41st International Conference on Machine Learning}, volume 235 of \emph{Proceedings of Machine Learning Research}, pages 34846--34865. PMLR, 2024.

\bibitem[Cini et~al.(2023{\natexlab{a}})Cini, Marisca, Zambon, and Alippi]{cini2023taming}
Andrea Cini, Ivan Marisca, Daniele Zambon, and Cesare Alippi.
\newblock Taming local effects in graph-based spatiotemporal forecasting.
\newblock In \emph{Advances in Neural Information Processing Systems}, 2023{\natexlab{a}}.

\bibitem[Bamberger et~al.(2025)Bamberger, Gutteridge, Roux, Bronstein, and Dong]{bamberger2025measuring}
Jacob Bamberger, Benjamin Gutteridge, Scott~Le Roux, Michael~M. Bronstein, and Xiaowen Dong.
\newblock On measuring long-range interactions in graph neural networks.
\newblock In \emph{Proceedings of the 42nd International Conference on Machine Learning}, volume 267 of \emph{Proceedings of Machine Learning Research}, pages 2770--2789. PMLR, 13--19 Jul 2025.

\bibitem[Di~Giovanni et~al.(2024)Di~Giovanni, Rusch, Bronstein, Deac, Lackenby, Mishra, and Veli{\v{c}}kovi{\'c}]{di_giovanni2024how}
Francesco Di~Giovanni, T.~Konstantin Rusch, Michael Bronstein, Andreea Deac, Marc Lackenby, Siddhartha Mishra, and Petar Veli{\v{c}}kovi{\'c}.
\newblock How does over-squashing affect the power of gnns?
\newblock \emph{Transactions on Machine Learning Research}, 2024.
\newblock ISSN 2835-8856.

\bibitem[Yu et~al.(2018)Yu, Yin, and Zhu]{yu2018spatio}
Bing Yu, Haoteng Yin, and Zhanxing Zhu.
\newblock Spatio-temporal graph convolutional networks: a deep learning framework for traffic forecasting.
\newblock In \emph{Proceedings of the 27th International Joint Conference on Artificial Intelligence}, pages 3634--3640, 2018.

\bibitem[LeCun et~al.(1995)LeCun, Bengio, et~al.]{lecun1995convolutional}
Yann LeCun, Yoshua Bengio, et~al.
\newblock Convolutional networks for images, speech, and time series.
\newblock \emph{The handbook of brain theory and neural networks}, 3361\penalty0 (10):\penalty0 1995, 1995.

\bibitem[Satorras et~al.(2022)Satorras, Rangapuram, and Januschowski]{satorras2022multivariate}
Victor~Garcia Satorras, Syama~Sundar Rangapuram, and Tim Januschowski.
\newblock Multivariate time series forecasting with latent graph inference.
\newblock \emph{arXiv preprint arXiv:2203.03423}, 2022.

\bibitem[Cini et~al.(2023{\natexlab{b}})Cini, Zambon, and Alippi]{cini2023sparse}
Andrea Cini, Daniele Zambon, and Cesare Alippi.
\newblock Sparse graph learning from spatiotemporal time series.
\newblock \emph{Journal of Machine Learning Research}, 24\penalty0 (242):\penalty0 1--36, 2023{\natexlab{b}}.

\bibitem[Cini et~al.(2023{\natexlab{c}})Cini, Marisca, Bianchi, and Alippi]{cini2023scalable}
Andrea Cini, Ivan Marisca, Filippo~Maria Bianchi, and Cesare Alippi.
\newblock Scalable spatiotemporal graph neural networks.
\newblock \emph{Proceedings of the AAAI Conference on Artificial Intelligence}, 37\penalty0 (6):\penalty0 7218--7226, June 2023{\natexlab{c}}.
\newblock \doi{10.1609/aaai.v37i6.25880}.

\bibitem[Gravina et~al.(2025)Gravina, Eliasof, Gallicchio, Bacciu, and Schönlieb]{gravina2025oversquashing}
Alessio Gravina, Moshe Eliasof, Claudio Gallicchio, Davide Bacciu, and Carola-Bibiane Schönlieb.
\newblock On oversquashing in graph neural networks through the lens of dynamical systems.
\newblock \emph{Proceedings of the AAAI Conference on Artificial Intelligence}, 39\penalty0 (16):\penalty0 16906--16914, April 2025.
\newblock \doi{10.1609/aaai.v39i16.33858}.

\bibitem[Krim et~al.(2022)Krim, Mezeddek, and Smail]{krim2022powers}
Ismaiel Krim, Mohamed~Tahar Mezeddek, and Abderrahmane Smail.
\newblock On powers and roots of triangular toeplitz matrices.
\newblock \emph{Applied Mathematics E-Notes}, 22:\penalty0 322--330, 2022.

\bibitem[Xiao et~al.(2024)Xiao, Tian, Chen, Han, and Lewis]{xiao2023efficient}
Guangxuan Xiao, Yuandong Tian, Beidi Chen, Song Han, and Mike Lewis.
\newblock Efficient streaming language models with attention sinks.
\newblock In \emph{International Conference on Learning Representations}, 2024.

\bibitem[Barbero et~al.(2024{\natexlab{a}})Barbero, Banino, Kapturowski, Kumaran, Madeira~Ara{\'u}jo, Vitvitskyi, Pascanu, and Veli{\v{c}}kovi{\'c}]{barbero2024transformers}
Federico Barbero, Andrea Banino, Steven Kapturowski, Dharshan Kumaran, Jo{\~a}o Madeira~Ara{\'u}jo, Oleksandr Vitvitskyi, Razvan Pascanu, and Petar Veli{\v{c}}kovi{\'c}.
\newblock Transformers need glasses! information over-squashing in language tasks.
\newblock \emph{Advances in Neural Information Processing Systems}, 37:\penalty0 98111--98142, 2024{\natexlab{a}}.

\bibitem[Kipf and Welling(2017)]{kipf2017semisupervised}
Thomas~N. Kipf and Max Welling.
\newblock Semi-supervised classification with graph convolutional networks.
\newblock In \emph{International {{Conference}} on {{Learning Representations}}}, 2017.

\bibitem[Atwood and Towsley(2016)]{atwood2016diffusion}
James Atwood and Don Towsley.
\newblock Diffusion-convolutional neural networks.
\newblock In \emph{Advances in Neural Information Processing Systems}, volume~29. Curran Associates, Inc., 2016.

\bibitem[Bresson and Laurent(2017)]{bresson2017residual}
Xavier Bresson and Thomas Laurent.
\newblock Residual gated graph convnets.
\newblock \emph{arXiv preprint arXiv:1711.07553}, 2017.

\bibitem[Karhadkar et~al.(2023)Karhadkar, Banerjee, and Montufar]{karhadkar2023fosr}
Kedar Karhadkar, Pradeep~Kr. Banerjee, and Guido Montufar.
\newblock Fosr: First-order spectral rewiring for addressing oversquashing in gnns.
\newblock In \emph{The Eleventh International Conference on Learning Representations}, 2023.

\bibitem[Schlichtkrull et~al.(2018)Schlichtkrull, Kipf, Bloem, van den Berg, Titov, and Welling]{schlichtkrull2018modeling}
Michael Schlichtkrull, Thomas~N. Kipf, Peter Bloem, Rianne van den Berg, Ivan Titov, and Max Welling.
\newblock Modeling relational data with graph convolutional networks.
\newblock In \emph{The Semantic Web}, pages 593--607, Cham, 2018. Springer International Publishing.
\newblock ISBN 978-3-319-93417-4.

\bibitem[Gasteiger et~al.(2019)Gasteiger, Wei\ss~enberger, and G\"{u}nnemann]{gasteiger2019diffusion}
Johannes Gasteiger, Stefan Wei\ss~enberger, and Stephan G\"{u}nnemann.
\newblock Diffusion improves graph learning.
\newblock In \emph{Advances in Neural Information Processing Systems}, volume~32. Curran Associates, Inc., 2019.

\bibitem[Barbero et~al.(2024{\natexlab{b}})Barbero, Velingker, Saberi, Bronstein, and Giovanni]{barbero2024locality}
Federico Barbero, Ameya Velingker, Amin Saberi, Michael~M. Bronstein, and Francesco~Di Giovanni.
\newblock Locality-aware graph rewiring in gnns.
\newblock In \emph{The Twelfth International Conference on Learning Representations}, 2024{\natexlab{b}}.

\bibitem[Chen et~al.(2020)Chen, Lin, Li, Li, Zhou, and Sun]{chen2020measuring}
Deli Chen, Yankai Lin, Wei Li, Peng Li, Jie Zhou, and Xu~Sun.
\newblock Measuring and relieving the over-smoothing problem for graph neural networks from the topological view.
\newblock \emph{Proceedings of the AAAI Conference on Artificial Intelligence}, 34\penalty0 (04):\penalty0 3438--3445, April 2020.
\newblock \doi{10.1609/aaai.v34i04.5747}.

\bibitem[Tortorella and Micheli(2022)]{tortorella2022leave}
Domenico Tortorella and Alessio Micheli.
\newblock Leave graphs alone: Addressing over-squashing without rewiring.
\newblock In \emph{The First Learning on Graphs Conference}, 2022.

\bibitem[Gravina et~al.(2023)Gravina, Bacciu, and Gallicchio]{gravina2023anti}
Alessio Gravina, Davide Bacciu, and Claudio Gallicchio.
\newblock Anti-symmetric dgn: a stable architecture for deep graph networks.
\newblock In \emph{The Eleventh International Conference on Learning Representations}, 2023.

\bibitem[Vorontsov et~al.(2017)Vorontsov, Trabelsi, Kadoury, and Pal]{vorontsov2017orthogonality}
Eugene Vorontsov, Chiheb Trabelsi, Samuel Kadoury, and Chris Pal.
\newblock On orthogonality and learning recurrent networks with long term dependencies.
\newblock In \emph{Proceedings of the 34th International Conference on Machine Learning}, volume~70 of \emph{Proceedings of Machine Learning Research}, pages 3570--3578. PMLR, 06--11 Aug 2017.

\bibitem[Chang et~al.(2021)Chang, Chen, Haber, and Chi]{chang2019antisymmetricrnn}
Bo~Chang, Minmin Chen, Eldad Haber, and Ed~H Chi.
\newblock Antisymmetricrnn: A dynamical system view on recurrent neural networks.
\newblock In \emph{International Conference on Learning Representations}, 2021.

\bibitem[Rusch and Mishra(2021)]{rusch2021unicornn}
T.~Konstantin Rusch and Siddhartha Mishra.
\newblock Unicornn: A recurrent model for learning very long time dependencies.
\newblock In \emph{Proceedings of the 38th International Conference on Machine Learning}, volume 139 of \emph{Proceedings of Machine Learning Research}, pages 9168--9178. PMLR, 18--24 Jul 2021.

\bibitem[Wu et~al.(2025)Wu, Wang, Jegelka, and Jadbabaie]{wu2025emergence}
Xinyi Wu, Yifei Wang, Stefanie Jegelka, and Ali Jadbabaie.
\newblock On the emergence of position bias in transformers.
\newblock \emph{arXiv preprint arXiv:2502.01951}, 2025.

\bibitem[Gravina et~al.(2024)Gravina, Lovisotto, Gallicchio, Bacciu, and Grohnfeldt]{gravina2024long}
Alessio Gravina, Giulio Lovisotto, Claudio Gallicchio, Davide Bacciu, and Claas Grohnfeldt.
\newblock Long range propagation on continuous-time dynamic graphs.
\newblock In \emph{Proceedings of the 41st International Conference on Machine Learning}, volume 235 of \emph{Proceedings of Machine Learning Research}, pages 16206--16225. PMLR, 21--27 Jul 2024.

\bibitem[Grinstead and Snell(2006)]{grinstead2006grinstead}
Charles~Miller Grinstead and James~Laurie Snell.
\newblock \emph{Grinstead and Snell's introduction to probability}.
\newblock Chance Project, 2006.

\bibitem[Van~Rossum and Drake(2009)]{rossum2009python}
Guido Van~Rossum and Fred~L. Drake.
\newblock \emph{Python 3 Reference Manual}.
\newblock CreateSpace, Scotts Valley, CA, 2009.
\newblock ISBN 1441412697.

\bibitem[Paszke et~al.(2019)Paszke, Gross, Massa, Lerer, Bradbury, Chanan, Killeen, Lin, Gimelshein, Antiga, Desmaison, Kopf, Yang, DeVito, Raison, Tejani, Chilamkurthy, Steiner, Fang, Bai, and Chintala]{paske2019pytorch}
Adam Paszke, Sam Gross, Francisco Massa, Adam Lerer, James Bradbury, Gregory Chanan, Trevor Killeen, Zeming Lin, Natalia Gimelshein, Luca Antiga, Alban Desmaison, Andreas Kopf, Edward Yang, Zachary DeVito, Martin Raison, Alykhan Tejani, Sasank Chilamkurthy, Benoit Steiner, Lu~Fang, Junjie Bai, and Soumith Chintala.
\newblock Pytorch: An imperative style, high-performance deep learning library.
\newblock In \emph{Advances in Neural Information Processing Systems 32}, pages 8024--8035. Curran Associates, Inc., 2019.

\bibitem[Fey and Lenssen(2019)]{fey2019fast}
Matthias Fey and Jan~E. Lenssen.
\newblock Fast graph representation learning with pytorch geometric.
\newblock In \emph{ICLR Workshop on Representation Learning on Graphs and Manifolds}, 2019.

\bibitem[Cini and Marisca(2022)]{Cini_Torch_Spatiotemporal_2022}
Andrea Cini and Ivan Marisca.
\newblock Torch spatiotemporal, March 2022.
\newblock URL \url{https://github.com/TorchSpatiotemporal/tsl}.

\bibitem[Falcon and {The PyTorch Lightning team}(2019)]{Falcon_PyTorch_Lightning_2019}
William Falcon and {The PyTorch Lightning team}.
\newblock Pytorch lightning, March 2019.
\newblock URL \url{https://github.com/PyTorchLightning/pytorch-lightning}.

\bibitem[Yadan(2019)]{Yadan2019Hydra}
Omry Yadan.
\newblock Hydra - a framework for elegantly configuring complex applications.
\newblock Github, 2019.
\newblock URL \url{https://github.com/facebookresearch/hydra}.

\bibitem[Harris et~al.(2020)Harris, Millman, Van Der~Walt, Gommers, Virtanen, Cournapeau, Wieser, Taylor, Berg, Smith, et~al.]{harris2020array}
Charles~R Harris, K~Jarrod Millman, St{\'e}fan~J Van Der~Walt, Ralf Gommers, Pauli Virtanen, David Cournapeau, Eric Wieser, Julian Taylor, Sebastian Berg, Nathaniel~J Smith, et~al.
\newblock Array programming with numpy.
\newblock \emph{Nature}, 585\penalty0 (7825):\penalty0 357--362, 2020.

\bibitem[Biewald(2020)]{wandb}
Lukas Biewald.
\newblock Experiment tracking with weights and biases, 2020.
\newblock URL \url{https://www.wandb.com/}.
\newblock Software available from wandb.com.

\bibitem[Kingma and Ba(2014)]{kingma2014adam}
Diederik~P Kingma and Jimmy Ba.
\newblock Adam: A method for stochastic optimization.
\newblock \emph{arXiv preprint arXiv:1412.6980}, 2014.

\bibitem[Hendrycks and Gimpel(2016)]{hendrycks2016gaussian}
Dan Hendrycks and Kevin Gimpel.
\newblock Gaussian error linear units (gelus).
\newblock \emph{arXiv preprint arXiv:1606.08415}, 2016.

\end{thebibliography}
\bibliographystyle{unsrtnat}

\clearpage
\appendix
\onecolumn
\section*{Appendix}
\section{Proofs}
\label{app:proofs}

This appendix gathers the complete proofs of all theoretical results showcased in the main text. For clarity and ease of reference, we restate each proposition or theorem before providing its corresponding proof. We begin with the results pertaining to \glspl{tcn}, followed by those related to the sensitivity analysis in \glspl{mptcn}.

\subsection{Sensitivity bound of \glspl{tcn}}
\label{app:tempsqu_proof}

\autoref{thm:tcnbound} establishes the sensitivity bound for a \gls{tcn} composed by \(L_{\mathsf T}\) stacked causal convolutional layers $\tcnl$. The resulting inequality factorises into a model-dependent term \((c_\sigma\mathsf w)^{L_{\mathsf T}}\) and a term \(\rmR^{L_{\mathsf T}}\) dependent instead on the temporal topology. The proof proceeds by induction on the number of layers, bootstrapping from the single–layer Jacobian estimate we provide in \autoref{lem:singletcn}.

\begin{lemma}[Single $\tcnl$ layer]\label{lem:singletcn} Consider a $\tcnl$ layer as in \autoref{eq:tcn_ref} with kernel size $P$, and assume that $\norm{\mW_{p}} \leq  \mathsf{w}$ for all $p < P$, and that $|\sigma^\prime| \leq c_{\sigma}$. For each $i, j \in [0, T)$, the following holds:
  \[
    \norm{\frac{\partial \vh^{(1)}_{t \shortminus j}}{\partial \vh^{(0)}_{t \shortminus i}}} \leq c_\sigma \mathsf{w} \left(\rmR\right)_{ij} .
  \]
\end{lemma}

\begin{proof}
  Let $\tilde{\vh}^{(1)}_{t}$ be the pre-activation output of the $\tcnl$ layer, such that $\vh^{(1)}_{t} = \sigma \big( \tilde{\vh}^{(1)}_{t} \big)$. Since $\rmR$ is lower-triangular and Toeplitz and has lower bandwidth $P-1$, i.e., \( \left(\rmR\right)_{ij} = r_{i-j} \) for \( 0 \le i-j < P \), for indices $i, j \in [0, T)$, we have
  \[
    \frac{\partial \tilde{\vh}^{(1)}_{t\shortminus j}}%
    {\partial \vh^{(0)}_{t\shortminus i}} =
    \begin{cases}
      r_{i-j} \mW_{i-j},
         & 0 \le i-j < P,    \\
      0, & \text{otherwise}.
    \end{cases}
  \]
  Applying the chain rule through the pointwise non-linearity $\sigma$
  gives
  \[
    \frac{\partial \vh^{(1)}_{t\shortminus j}}%
    {\partial \vh^{(0)}_{t\shortminus i}}
    =
    \diag \Big(
    \sigma^\prime\Big(\tilde{\vh}^{(1)}_{t\shortminus j}\Big)
    \Big)
    \frac{\partial \tilde{\vh}^{(1)}_{t\shortminus j}}%
    {\partial \vh^{(0)}_{t\shortminus i}}.
  \]
  Applying the sub-multiplicative property of the spectral norm together with the bounds $\lvert\sigma^\prime\rvert\le c_{\sigma}$ and $\norm{\mW_{p}} \le\mathsf{w}$ yields
  \begin{align*}
    \norm{\frac{\partial \vh^{(1)}_{t\shortminus j}}{\partial \vh^{(0)}_{t\shortminus i}}} & \le \norm{\diag \Big(\sigma^\prime\Big(\tilde{\vh}^{(1)}_{t\shortminus j}\Big)\Big)} \norm{\frac{\partial \tilde{\vh}^{(1)}_{t\shortminus j}}{\partial \vh^{(0)}_{t\shortminus i}}} \\
    & \le c_{\sigma} \mathsf{w} |r_{i-j}| \\
    &= c_{\sigma} \mathsf{w} \left(\rmR\right)_{ij},
  \end{align*}
  which proves the theorem.
\end{proof}

We now use the result in \autoref{lem:singletcn} to prove the bound of \autoref{thm:tcnbound} for a \gls{tcn} obtained by stacking $L_{\mathsf{T}}$ $\tcnl$ layers.

\tcnbound*

\begin{proof}
  We fix the number of stacked $\tcnl$ layers to be $l = L_{\mathsf{T}}$ and prove the bound by induction on $l$. For the base case $l=1$, \autoref{lem:singletcn} gives
  \begin{equation}
    \norm{\frac{\partial \vh^{(1)}_{t \shortminus j}}{\partial \vh^{(0)}_{t \shortminus i}}} \le c_\sigma\mathsf{w}(\rmR)_{ij} = (c_\sigma \mathsf{w})^1(\rmR^{1})_{ij},
    \label{eq:thm_tcnbound_1}
    \tag{\ref*{thm:tcnbound}.1}
  \end{equation}
  so the claim holds for one layer. We assume the bound is true for $l-1$:
  \begin{equation}
    \norm{\frac{\partial \vh^{(l \shortminus 1)}_{t \shortminus k}}{\partial \vh^{(0)}_{t \shortminus i}}} \le (c_\sigma\mathsf w)^{l \shortminus 1}(\rmR^{l \shortminus 1})_{ik} \qquad \forall\, i,k.
    \label{eq:thm_tcnbound_ih}
    \tag{\ref*{thm:tcnbound}.2}
  \end{equation}

  For $l$ layers, the chain rule gives
  \[
    \norm{\frac{\partial \vh^{(l)}_{t \shortminus j}}{\partial \vh^{(0)}_{t \shortminus i}}} =
    \norm{
      \sum_{k=0}^{T-1}
      \frac{\partial \vh^{(l)}_{t \shortminus j}}{\partial \vh^{(l \shortminus 1)}_{t \shortminus k}}
      \frac{\partial \vh^{(l \shortminus 1)}_{t \shortminus k}}{\partial \vh^{(0)}_{t \shortminus i}}
    }
    \le
    \sum_{k=0}^{T-1}
    \norm{
      \frac{\partial \vh^{(l)}_{t \shortminus j}}{\partial \vh^{(l \shortminus 1)}_{t \shortminus k}}
    }
    \norm{\frac{\partial \vh^{(l \shortminus 1)}_{t \shortminus k}}{\partial \vh^{(0)}_{t \shortminus i}}},
  \]
  where the first factor is the base case of a single layer \eqref{eq:thm_tcnbound_1}, and the second one is the induction hypothesis \eqref{eq:thm_tcnbound_ih}. Substituting these bounds gives
  \begin{align*}
    \norm{\frac{\partial \vh^{(l)}_{t \shortminus j}}{\partial \vh^{(0)}_{t \shortminus i}}} %
     & \le \sum_{k=0}^{T-1} \left( (c_\sigma\mathsf w) (\rmR)_{kj} (c_\sigma\mathsf w)^{l \shortminus 1} (\rmR^{l \shortminus 1})_{ik} \right) \\ %
     & = (c_\sigma\mathsf w)^{l} \sum_{k=0}^{T-1} (\rmR^{l \shortminus 1})_{ik} (\rmR)_{kj} \\ %
     & = (c_\sigma\mathsf w)^{l}(\rmR^{l})_{ij} ,
  \end{align*}
  which proves the induction.
\end{proof}

\subsection{Asymptotic of $\rmR$ and sensitivity in \acrshortpl{mptcn}}
\label{app:mptcn_proof}

In the following, we provide a detailed analysis of the asymptotic behavior of the temporal topology matrix, which characterizes information propagation across time steps in both \glspl{tcn} and \glspl{mptcn}. Afterwards, we provide the sensitivity bound and related proof for the latter.

\subsubsection{Asymptotic behavior of $\rmR$ and $\rmR_N$}

\tempsqu*

\begin{proof}
  Write $\rmR = r_0 \rmI + \rmN$, where $\rmN := \rmR-r_0\rmI$ is strictly lower-triangular. Because $\rmI$ and $\rmN$ commute, the binomial expansion for commuting matrices gives
  \[
    \rmR^{l}
    \;=\;
    r_0^{l}\rmI
    \;+\;
    \sum_{k=1}^{T-1} \binom{l}{k} r_0^{l-k}\rmN^{k},
    \qquad l \in \sN,
  \]
  the sum truncating at $T-1$ since $\rmN^{T}= \vzero$. For a fixed row index $i \ge 1$, the strictly lower-triangular structure implies $(\rmN^{k})_{i0}=0$ whenever $k>i$, thus
  \[
    (\rmR^{l})_{i0}
    \;=\;
    \sum_{k=1}^{i} \binom{l}{k} r_0^{l-k} (\rmN^{k})_{i0}.
  \]
  With $k$ fixed, $\binom{l}{k} \sim l^{k}/k!$ as $l\to\infty$ since $\lim_{l\rightarrow \infty } \frac{\binom{l}{k}}{l^k} = \frac{1}{k!}$, so with some abuse of notation
  \[
    (\rmR^l)_{i0} \sim \sum_{k=1}^{i} \frac{l^k}{k!} r_0^{l-k} (\rmN^k)_{i0} = r_0^{l} \sum_{k=1}^{i} \frac{l^k}{k! r_0^k} (\rmN^k)_{i0} ,
    \qquad l\to\infty .
  \]

    Note that it is the product of $r_0^l$ and a degree $i$ polynomial in $l$, therefore in order to see the behaviour as $l\rightarrow \infty$ it suffices to study the leading term $\frac{l^i}{i!r_0^i} (\rmN^i)_{i0}$ and the factor of $r_0^l$. Additionally, due to the strictly lower triangular structure of $\rmN$, we have $(\rmN^i)_{i0} = r_1^i$. This follows from the fact that, in a directed acyclic graph as the one described by $\rmN$, there exists a unique directed path of length $i$ from the $i$-th node to the sink node $0$. Each edge along this path contributes a multiplicative factor of $r_1$, resulting in a total weight of $r_1^i$ for the path.
    Applying the same reasoning with $i$ replaced by $j<i$, we can study the ratio
  \[
    \left|\frac{(\rmR^{l})_{j0}}{(\rmR^{l})_{i0}}\right|
    \sim
    \frac{\frac{l^j}{j!}r_0^{l-j}r_1^j}{\frac{l^i}{i!}r_0^{l-i}r_1^i}
    =
    \frac{1}{l^{i-j}} \frac{i! r_0^{i-j}}{j! r_1^{i-j}}
    =
    \mathcal{O}\left( \frac{1}{l^{i-j}} \right)
    \xrightarrow[l\to\infty]{} 0,
  \]
  which proves the proposition.
\end{proof}

The proposition that follows shows that the powers \(\rmR_{N}^{k}\) converge to the rank–one matrix \(\vone\ve_{0}^{\top}\). In probabilistic terms, the normalization turns the temporal topology into an absorbing Markov chain whose unique absorbing state is the last time step. Consequently, the influence of any time step $t-i$ concentrates on $t$ as the depth grows, eventually reaching a uniform sensitivity of the last time step over the entire sequence. \autoref{fig:tcn_sensitivity} shows a comparison between the sensitivity curves as a function of the (backward) temporal distance from the last time step for increasingly deeper \glspl{tcn} in the case of standard (a) and normalized (b) convolutions.

\begin{figure}[t]
    \centering
    \includegraphics[width=\linewidth]{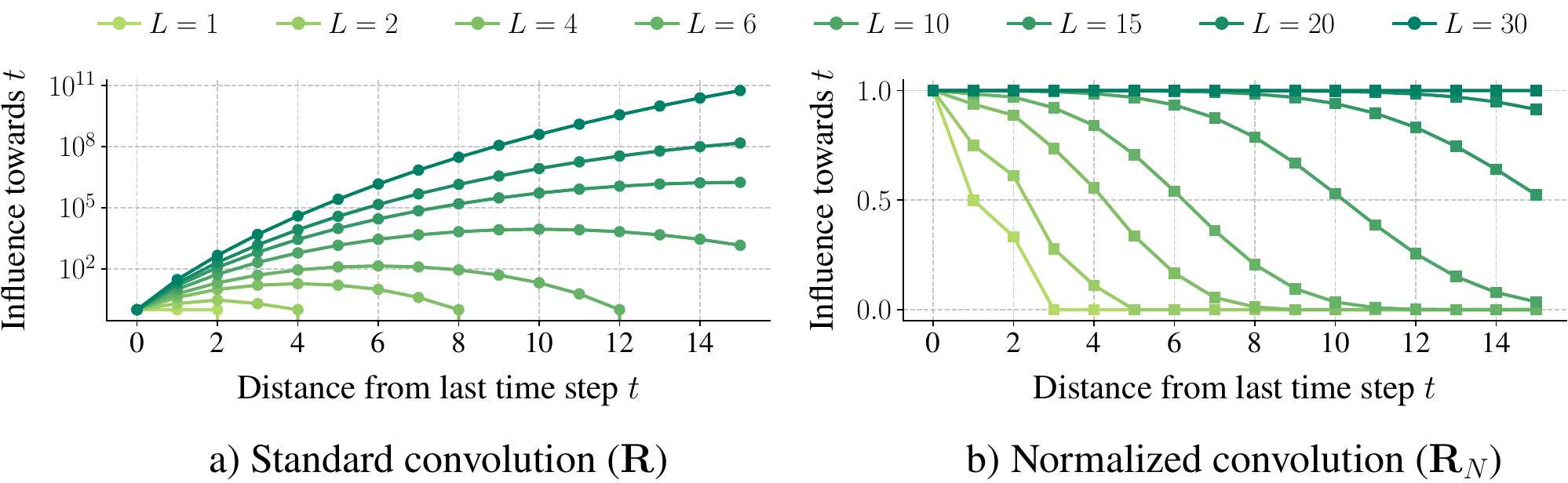}
    \caption{Sensitivity of last-layer representations associated with last time step $t$ to earlier ones in \glspl{tcn} with $L$ layers and kernel size $P=3$. The values correspond to entries $\left(\rmR^L\right)_{i0}$ for the standard convolution (a) and $\left(\rmR_N^L\right)_{i0}$ for the normalized convolution (b), with $i \ge 0$ being the backward distance from $t$. As depth increases, the standard convolution favors information from earlier steps, while the normalized version asymptotically approaches uniform sensitivity across all steps.}
    \label{fig:tcn_sensitivity}
\end{figure}

\begin{proposition}
  \label{prop:rn_asympt}
  Let $\rmR\in\sR_{\scalemath{.65}{\scalemath{.85}{\geq} 0}}^{T\times T}$ be a positive, real, lower–triangular Toeplitz band matrix with lower bandwidth $P-1$, i.e., with \( \left(\rmR\right)_{ij} = r_{i-j} \) for \( 0 \le i-j < P \), and let $P\geq 2$, $r_1\neq 0$, and $r_0\neq 0$. We define $\rmR_{N}:=\diag(\rmR\vone)^{-1}\rmR$ to be the row–normalized matrix $\rmR$ ($\vone$ is the all–ones vector). Then
  \[
    \lim_{k\to\infty} \rmR_{N}^{k} = \vone \ve_{0}^{\top},
  \]
  i.e. every row of $\rmR_{N}^{k}$ converges to $\ve_{0}^{\top}=(1,0,\dots,0)$.
\end{proposition}

\begin{proof}
  Since $\rmR_{N}$ is a stochastic matrix (every row sum to $1$ by construction), we can interpret $\rmR_{N}$ as the transition matrix of a Markov chain on the finite state space $\{0,1,\dots,T-1\}$. Because $\rmR$ is lower-triangular, we have that state $0$ is \emph{absorbing}, i.e.
  \(
    (\rmR_{N})_{00} = r_{0}/r_{0} = 1
  \),
  hence once the chain reaches $0$ it never leaves. Moreover, since $r_{1}\neq0$, for every $i>0$ and $j=i-1$ there is a positive probability of moving one step closer to $0$, i.e.,
  \[
    (\rmR_{N})_{ij} = \frac{r_{1}}{\sum_{p=0}^{\min\{P \shortminus 1,i\}}r_{p}} > 0 ,
  \]
  hence all states $1,2,\dots,T-1$ are \emph{transient}. For a finite absorbing Markov chain with a single absorbing state $\{0\}$, every state eventually reaches the absorbing state $0$ with probability $1$ (Theorem 11.3 from \cite{grinstead2006grinstead}). Thus,
  \[
    \lim_{k\to\infty} \rmR_{N}^{k} = \vone \ve_{0}^{\top} ,
  \]
  where the $i$-th row of the limit contains the absorption probabilities starting from state $i$ and are all equal to $1$ for state $0$ and to $0$ for the others.
\end{proof}

\subsubsection{Sensitivity bounds of \acrshortpl{mptcn}}
In this subsection, we aim to prove the sensitivity bounds of a \gls{mptcn} obtained by stacking $L$ $\stmpl$ layers defined as in \autorefseq{eq:tmp_disjoint}{eq:mp_disjoint}, where the temporal processing $\tmpl$ takes the form of a temporal causal convolution, denoted by $\tcnl$, as defined in \autoref{eq:tcn_ref}.

In line with previous work by~\citet{di_giovanni2024how}, for spatial processing, we consider as $\mpl$ the following family of \glspl{mpnn}:
\begin{equation}
    \label{eq:mp_ref}
    \vh^{v (l)} = \xi\left(\Theta_{\mathsf{U}}^{(l)} \vh^{v(l \shortminus 1)} + \Theta_{\mathsf{M}}^{(l)} \sum\nolimits_{u \in \gN(v)} \gso^{uv} \phi^{(l)} \left(\vh^{v(l \shortminus 1)}, \vh^{u(l \shortminus 1)}\right) \right)
\end{equation}
where $\Theta_{\mathsf{U}}^{(l)}, \Theta_{\mathsf{M}}^{(l)} \in \sR^{d \times d}$ are matrices of learnable weights, $\phi^{(l)}$ is a $\gC^1$ function, and $\xi$ is a pointwise nonlinear activation function. This class includes common \glspl{mpnn}, such as GCN~\citep{kipf2017semisupervised}, DCNN~\citep{atwood2016diffusion},
GIN~\citep{xu2019powerful},
and GatedGCN~\citep{bresson2017residual}. For this class of function, we make the following assumptions:

\begin{assumption}
  \label{ass:mpnn}
  Given an \gls{mpnn} with $L$ layers, each as in \autoref{eq:mp_ref}, we assume for each layer $l$ that $|\xi'|\leq c_{\xi}$,
  $\norm{\Theta_{\mathsf{U}}^{(l)}} \leq \theta_{\mathsf{u}}$, and $\norm{\Theta_{\mathsf{M}}^{(l)}}\leq \theta_{\mathsf{m}}$. We further assume the Jacobian of the message function $\phi^{(l)}$ w.r.t.\ the target ($v$) and neighbor ($u$) node features to be bounded as $\norm{\partial \phi^{(l)} / \partial \vh^{v(l \shortminus 1)}}\leq c_{1}$ and $\norm{\partial \phi^{(l)} / \partial \vh^{u(l \shortminus 1)}}\leq c_{2}$.
\end{assumption}

Given an \gls{mpnn} as defined in \autoref{eq:mp_ref} and for which \autoref{ass:mpnn} holds, \citet{di_giovanni2024how} established the following sensitivity bound.

\mpnnspec*

As done in the previous section to prove the bound for \glspl{tcn}, we proceed by induction on the number of layers $L$, and start the analysis by establishing the bound for a single-layer \gls{mptcn} in \autoref{thm:mptcn_tts}. The proof takes advantage of the result for \glspl{tcn} we demonstrated in the previous section \autorefp{thm:tcnbound} and  \autoref{thm:mpnnspec} by \citet{di_giovanni2024how} for \glspl{mpnn}.

\begin{lemma}[Sensitivity bound \gls{tts} \gls{mptcn}]
  \label{thm:mptcn_tts}%
  Consider a \gls{tts} \acrshort{mptcn} ($L=1$) with $L_{\mathsf{T}}$ temporal ($\tmpl$) layers and $L_{\mathsf{S}}$ spatial ($\mpl$) as defined in \autorefseq{eq:tmp_disjoint}{eq:mp_disjoint}. Assume that each $\tmpl$ layer satisfies the conditions of \autoref{thm:tcnbound}, and each $\mpl$ layer satisfies the assumptions in \autoref{thm:mpnnspec}. Then, for any $v, u \in \gV$ and $i,j \in [0, T)$, the following holds:
  \[
    \norm{\frac{\partial \vh^{v(1)}_{t \shortminus j}}{\partial \vh^{u(0)}_{t \shortminus i}}} \leq %
    \underbrace{\left(c_{\xi} \theta_{\mathsf{m}}\right)^{L_{\mathsf{S}}} \left(\rmS^{L_{\mathsf{S}}}\right)_{uv}}_{\mathrm{space}} %
    \underbrace{\left(c_\sigma  \mathsf{w} \right)^{L_{\mathsf{T}}}\left(\rmR^{L_{\mathsf{T}}}\right)_{ij}}_{\mathrm{time}} .
  \]
\end{lemma}

\begin{proof}
  Fix \(u,v\in\gV\) and \(i,j\in[0,T)\).
  One $\stmpl$ layer first applies the \(L_{\mathsf T}\)  \(\tmpl\) layers node–wise, and then the
  \(L_{\mathsf S}\) \(\mpl\) layers time‐wise. Thus, for the chain rule, we have
  \begin{equation}
    \label{eq:thm_mptcn_tts}
    \frac{\partial \vh^{v(1)}_{t \shortminus j}}{\partial \vh^{u(0)}_{t \shortminus i}}
    = \sum_{w\in\gV}\sum_{k=0}^{T-1}
    \frac{\partial \vh^{v(1)}_{t \shortminus j}}{\partial \vz^{w(L_{\mathsf T})}_{t \shortminus k}}
    \frac{\partial \vz^{w(L_{\mathsf T})}_{t \shortminus k}}{\partial \vh^{u(0)}_{t \shortminus i}} .
    \tag{\ref*{thm:mptcn_tts}.1}
  \end{equation}

  Every \(\mpl\) layer processes each time step separately, hence
  \(\partial \vh^{v(1)}_{t \shortminus j}/\partial \vz^{w(L_{\mathsf T})}_{t \shortminus k}=0\)
  unless \(k=j\). Similarly, every \(\tmpl\) layer processes each node separately, hence
  \(\partial \vz^{w(L_{\mathsf T})}_{t \shortminus j}/\partial \vh^{u(0)}_{t \shortminus i}=0\)
  unless \(w=u\). Both sums in \autoref{eq:thm_mptcn_tts} therefore collapse, giving
  \[
    \frac{\partial \vh^{v(1)}_{t \shortminus j}}{\partial \vh^{u(0)}_{t \shortminus i}}
    = \frac{\partial \vh^{v(1)}_{t \shortminus j}}
    {\partial \vz^{u(L_{\mathsf T})}_{t \shortminus j}}
    \frac{\partial \vz^{u(L_{\mathsf T})}_{t \shortminus j}}
    {\partial \vh^{u(0)}_{t \shortminus i}} .
  \]

  Using sub-multiplicativity of the spectral norm, and considering the bounds from \autoref{thm:mpnnspec} for $\mpl$ layers and \autoref{thm:tcnbound} for $\tcnl$ layers, we have
  \begin{align*}
    \norm{\frac{\partial \vh^{v(1)}_{t \shortminus j}}{\partial \vh^{u(0)}_{t \shortminus i}}} & \le %
    \norm{\frac{\partial \vh^{v(1)}_{t \shortminus j}}
      {\partial \vz^{u(L_{\mathsf T})}_{t \shortminus j}}}
    \norm{\frac{\partial \vz^{u(L_{\mathsf T})}_{t \shortminus j}}
    {\partial \vh^{u(0)}_{t \shortminus i}}}                                                         \\
                                                                                               & \le
    \left(c_{\xi} \theta_{\mathsf{m}}\right)^{L_{\mathsf{S}}} \left(\rmS^{L_{\mathsf{S}}}\right)_{uv} %
    \left(c_\sigma  \mathsf{w} \right)^{L_{\mathsf{T}}}\left(\rmR^{L_{\mathsf{T}}}\right)_{ij} .
    \tag*{\qedhere}
  \end{align*}
\end{proof}

Building on this result for a single-layer \gls{mptcn}, we extend the sensitivity bound for a \gls{mptcn} of $L$ layers in the following theorem.

\mptcn*

\begin{proof}
  We fix the number of $\stmpl$ layers to be $l = L$ and prove the bound by induction on $l$. The base case where $l=1$ follows directly from \autoref{thm:mptcn_tts}, i.e.,
  \begin{equation}
    \norm{\frac{\partial \vh^{v(1)}_{t \shortminus j}}{\partial \vh^{u(0)}_{t \shortminus i}}} \le
    \left(c_{\xi} \theta_{\mathsf{m}}\right)^{L_{\mathsf{S}}} \big(\rmS^{L_{\mathsf{S}}}\big)_{uv} %
    \left(c_\sigma  \mathsf{w} \right)^{L_{\mathsf{T}}}\big(\rmR^{L_{\mathsf{T}}}\big)_{ij}
    \label{eq:thm_mptcn_1}
    \tag{\ref*{thm:mptcn}.1}
  \end{equation}

  We assume the bound holds for \(l-1\) stacked \(\stmpl\) blocks, for all $u,w \in \gV$ and $i,k \in [0, T)$:
  \begin{equation}
    \norm{
      \frac{\partial \vh^{w(l \shortminus 1)}_{t\shortminus k}}
      {\partial \vh^{u(0)}_{t\shortminus i}}
    }
    \le
    \left(c_{\xi} \theta_{\mathsf{m}}\right)^{(l \shortminus 1)L_{\mathsf{S}}}
    \left(c_\sigma  \mathsf{w} \right)^{(l \shortminus 1)L_{\mathsf{T}}}
    \big(\rmS^{(l \shortminus 1)L_{\mathsf{S}}}\big)_{uw}
    \big(\rmR^{(l \shortminus 1)L_{\mathsf{T}}}\big)_{ik}.
    \label{eq:thm_mptcn_ih}
    \tag{\ref*{thm:mptcn}.2}
  \end{equation}

  For the \(l\)-th layer, we apply the chain rule, the triangle inequality, and sub-multiplicativity:
  \[
    \norm{
      \frac{\partial \vh^{v(l)}_{t\shortminus j}}
      {\partial \vh^{u(0)}_{t\shortminus i}}
    }
    \le
    \sum_{w\in\gV}\sum_{k=0}^{T-1}
    \norm{
      \frac{\partial \vh^{v(l)}_{t\shortminus j}}
      {\partial \vh^{w(l \shortminus 1)}_{t\shortminus k}}
    }
    \norm{
      \frac{\partial \vh^{w(l \shortminus 1)}_{t\shortminus k}}
      {\partial \vh^{u(0)}_{t\shortminus i}}
    } .
  \]

  The first term is the base case of a single outer layer \eqref{eq:thm_mptcn_1}, and the right term is our induction hypothesis \eqref{eq:thm_mptcn_ih}, which combined with the triangular inequality proves the induction:

  \begin{align*}
    \norm{
      \frac{\partial \vh^{v(l)}_{t\shortminus j}}
      {\partial \vh^{u(0)}_{t\shortminus i}}
    }
     & \le \sum_{w\in\gV}\sum_{k=0}^{T-1} \left(
    \left(c_{\xi} \theta_{\mathsf{m}}\right)^{L_{\mathsf{S}}}
    \left(c_\sigma  \mathsf{w} \right)^{L_{\mathsf{T}}}
    \bigl(\rmS^{L_{\mathsf S}}\bigr)_{w v}
    \bigl(\rmR^{L_{\mathsf T}}\bigr)_{k j}
    \right.                                                                     \\
     & \hphantom{\le \sum_{w\in\gV}\sum_{k=0}^{T-1} \Big(} \left.
    \left(c_{\xi} \theta_{\mathsf{m}}\right)^{(l \shortminus 1)L_{\mathsf{S}}}
    \left(c_\sigma  \mathsf{w} \right)^{(l \shortminus 1)L_{\mathsf{T}}}
    \bigl(\rmS^{(l \shortminus 1)L_{\mathsf S}}\bigr)_{u w}
    \bigl(\rmR^{(l \shortminus 1)L_{\mathsf T}}\bigr)_{i k}\right)              \\
     & =
    \left(c_{\xi} \theta_{\mathsf{m}}\right)^{lL_{\mathsf{S}}}
    \left(c_\sigma  \mathsf{w} \right)^{lL_{\mathsf{T}}}
    \sum_{w\in\gV}\sum_{k=0}^{T-1}
    \bigl(\rmS^{L_{\mathsf S}}\bigr)_{w v}
    \bigl(\rmS^{(l \shortminus 1)L_{\mathsf S}}\bigr)_{u w}
    \bigl(\rmR^{L_{\mathsf T}}\bigr)_{k j}
    \bigl(\rmR^{(l \shortminus 1)L_{\mathsf T}}\bigr)_{i k}                     \\
     & =
    \left(c_{\xi} \theta_{\mathsf{m}}\right)^{lL_{\mathsf{S}}}
    \left(c_\sigma  \mathsf{w} \right)^{lL_{\mathsf{T}}}
    \bigl(\rmS^{(l \shortminus 1)L_{\mathsf S}}\rmS^{L_{\mathsf S}}\bigr)_{u v}
    \bigl(\rmR^{(l \shortminus 1)L_{\mathsf T}}\rmR^{L_{\mathsf T}}\bigr)_{i j} \\
     & =
    \left(c_{\xi} \theta_{\mathsf{m}}\right)^{lL_{\mathsf{S}}}
    \left(c_\sigma  \mathsf{w} \right)^{lL_{\mathsf{T}}}
    \bigl(\rmS^{lL_{\mathsf S}}\bigr)_{u v}
    \bigl(\rmR^{lL_{\mathsf T}}\bigr)_{i j}.
  \end{align*}

  Induction on $l$ establishes the inequality for every $l\le L$; in particular for $l = L$, which is exactly the claimed bound.
\end{proof}

\section{Experiments}
\label{app:experiments}

All numerical simulations are performed on regression tasks. For experiments on real-world data, the goal is $H$-steps-ahead forecasting. Given a window of $T$ past observations, the forecasting task consists of predicting the $H$ future observations at each node, i.e., $\vy^v_t = \vx_{t:t+H}^v$. We consider families of (parametric) models $f_{\vtheta}$ such that
\begin{equation}
  \hat{\vx}_{t:t+H}^v = \left(f_{\vtheta}\left(\mX_{t \shortminus T:t}, \mA, \mT \right)\right)_v , \quad \forall\ v \in \gV,
\end{equation}
where $\vtheta$ is the set of learnable parameters and $\hat{\vx}_{t:t+H}^v$ are the forecasted values at node $v$ for the interval $[t, t+H)$. The parameters are optimized using an element-wise cost function, e.g., the \gls{mse}:
\begin{equation}\label{eq:forecasting_loss}
  \hat{\vtheta} = \argmin_{\vtheta} \frac{1}{NH}\sum_{v \in \gV}\sum_{i=0}^{H-1} \norm{\hat{\vx}_{t+i}^v - \vx_{t+i}^v}_2^2 .
\end{equation}

\paragraph{Software \& Hardware}
\label{sec:software_harware}
All the code used for the experiments has been developed with Python~\cite{rossum2009python} and relies on the following open-source libraries: PyTorch~\cite{paske2019pytorch}; PyTorch Geometric~\cite{fey2019fast}; Torch Spatiotemporal~\cite{Cini_Torch_Spatiotemporal_2022}; PyTorch Lightning~\cite{Falcon_PyTorch_Lightning_2019}; Hydra~\cite{Yadan2019Hydra}; Numpy~\cite{harris2020array}. We relied on Weights \& Biases~\citep{wandb} for tracking and logging experiments. The code to reproduce the experiments is available at \href{https://github.com/marshka/spatiotemporal-oversquashing}{\texttt{github.com/marshka/spatiotemporal-oversquashing}}.

All experiments were conducted on a workstation running Ubuntu 22.04.5 LTS, equipped with two AMD EPYC 7513 CPUs and four NVIDIA RTX A5000 GPUs, each with 24 GB of memory. To accelerate the experimental process, multiple runs were executed in parallel, with shared access to both CPU and GPU resources. This setup introduces variability in execution times, even under identical experimental configurations. On average, experiments involving real-world datasets required approximately $1$ to $2$ hours per run, while synthetic experiments -- when not terminated early due to the early-stopping criterion -- completed within $20$ minutes.

\paragraph{Datasets}
\begin{table*}[b]
\caption{Statistics of the datasets and considered sliding-window parameters.}
\label{tab:datasets}
\setlength{\aboverulesep}{0pt}
\setlength{\belowrulesep}{0pt}
\renewcommand{\arraystretch}{1.25}
\centering
\begin{small}
\begin{tabular}{l|c c c c c||c c}
\toprule
 \bfseries Datasets & \bfseries Type & \bfseries Nodes & \bfseries Edges & \bfseries Time steps & \bfseries Sampling Rate & \bfseries Window & \bfseries Horizon\\
\toprule
\gls{la} & Directed & 207 & 1,515 & 34,272 & 5 minutes & 12 & 12 \\
\gls{bay} & Directed & 325 & 2,369 & 52,128 & 5 minutes & 12 & 12 \\
\gls{engrad} & Undirected & 487 & 2,297 & 26,304 & 1 hour & 24 & 6 \\
\bottomrule
\end{tabular}
\end{small}
\end{table*}

\begin{figure}[t]
     \centering
     \begin{subfigure}[b]{.48\textwidth}
         \centering
         \includegraphics[scale=1.5]{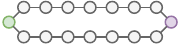}
         \vspace{.4cm}
         \caption{\textsc{Ring}}
         \label{fig:ring_graph}
     \end{subfigure}
     \begin{subfigure}[b]{.48\textwidth}
         \centering
         \includegraphics[scale=1.5]{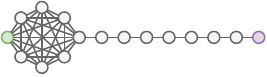}
         \caption{\textsc{Lollipop}}
         \label{fig:lollipop_graph}
     \end{subfigure}
    \caption{\textsc{Ring} and \textsc{Lollipop} graphs used in the synthetic experiments. We highlight in green the target node \protect\tikz\protect\node[vertex=drawiogreen] {};, and show an example of a source node \protect\tikz\protect\node[vertex=drawiopurple] {}; when the spatial distance $k$ is equal to the graph diameter.}
    \label{fig:synthetic graphs}
\end{figure}

We begin by introducing a set of synthetic datasets and tasks specifically designed to highlight the effects of over-squashing in both space and time.
\begin{description}[leftmargin=1em, itemindent=0em]
    \item[\gls{copyfirst}] %
    Each sequence $\vx_{t \shortminus T:t}$ consists of $T=16$ time steps, with values sampled uniformly from the $[0,1]$ interval. The task is to predict the first element in the sequence, i.e., $\vx_{t \shortminus T+1}$. We generate 20,000 sequences for training, 320 for validation, and 500 for testing.
    \item[\gls{copylast}] %
    This task is analogous to \gls{copyfirst}, but the model is required to predict the last value in the sequence, i.e., $\vx_t$. Together with \gls{copyfirst}, these are the datasets used in the experiments in \autoref{sec:ovs_tcn}.
    \item[\gls{rocketman}] %
    This is the dataset used in the synthetic experiments in \autoref{sec:ovs_mptcn}. We generate a graph with a given structure and $N=16$ nodes. At each node, we generate a sequence of $T=9$ time steps, again sampling values uniformly from the $[0,1]$ interval. The task is to predict, for a target node $v$, the average value at time step $t-i$ of nodes located exactly $k$ hops away from $v$. That is, given a spatial distance $k \in [0, D]$, where $D$ is the graph's diameter, we define the source set $\gN_k(v)$ as the nodes at shortest-path distance $k$ from $v$, with $\gN_0(v) = \{v\}$. The label for node $v$ is then $\vy^v_t = \sum_{u \in \gN_k(v)} \frac{\vx^u_{t \shortminus i}}{|\gN_k(v)|}$, while predictions for all other nodes are masked out. We use the same train/validation/test split as in the other synthetic tasks. We use as graphs the \textsc{Ring} and \textsc{Lollipop} graphs illustrated in \autoref{fig:synthetic graphs}. We further show in \autoref{fig:spacetime_app} the performance of an \gls{mptcn} in this dataset with both \gls{tts} (\ref{fig:spacetime_app_tts}) and \gls{tas} (\ref{fig:spacetime_app_tas}) approaches.
\end{description}

We now present the real-world datasets used in the experiments of \autoref{tab:realworld}. 
We split datasets into windows of $T$ time steps, and train the models to predict the next $H$ observations. We closely follow the setup of previous works \citep{marisca2024graph, cini2023taming}. In the following, we report detailed information for experiments on each dataset for completeness.
\begin{description}[leftmargin=1em, itemindent=0em]
    \item[\gls{la} \& \gls{bay}] Both are widely popular benchmarks for graph-based spatiotemporal forecasting. \gls{bay} contains $6$ months of data from $325$ traffic sensors in the San Francisco Bay Area, while \gls{la} contains $4$ months of analogous readings acquired from $207$ detectors in the Los Angeles County Highway~\citep{li2018diffusion}. We use the same setup as previous works~\cite{cini2023taming, wu2019graph, li2018diffusion} for all the preprocessing steps. As such, we normalize the target variable to have zero mean and unit variance on the training set and obtain the adjacency matrix as a thresholded Gaussian kernel on the road distances~\citep{li2018diffusion}. We sequentially split the windows into $70\%/10\%/20\%$ partitions for training, validation, and testing, respectively. We use encodings of the time of day and day of the week as additional input to the model. For \gls{la}, we impute the missing values with the last observed value and include a binary mask as an additional exogenous input.
    Window and horizon lengths are set as $T=12$ and $H=12$. 
    \item[\textbf{\gls{engrad}}] The \gls{engrad} dataset contains hourly measurements of $5$ different weather variables collected at $487$ grid points in England from 2018 to 2020. We use solar radiation as the target variable, while all other weather variables are used as additional inputs, along with encodings of the time of the day and the year. Window and horizon lengths are set as $T=24$ and $H=6$. We scale satellite radiation in the $[0, 1]$ range and normalize temperature values to have a zero mean and unit variance. We do not compute loss and metrics on time steps with zero radiance and follow the protocol of previous work~\citep{marisca2024graph} to obtain the graph and training/validation/testing folds.
\end{description}

\begin{figure}[t]
     \centering
     \begin{subfigure}[t]{\textwidth}
         \centering
        \includegraphics[width=\linewidth]{assets/img/plot_space_vs_time_tts.pdf}
         \caption{\textsc{\gls{tts} \gls{mptcn}}}
         \label{fig:spacetime_app_tts}
         \vspace{.2cm}
     \end{subfigure}
     \begin{subfigure}[t]{\textwidth}
         \centering
        \includegraphics[width=\linewidth]{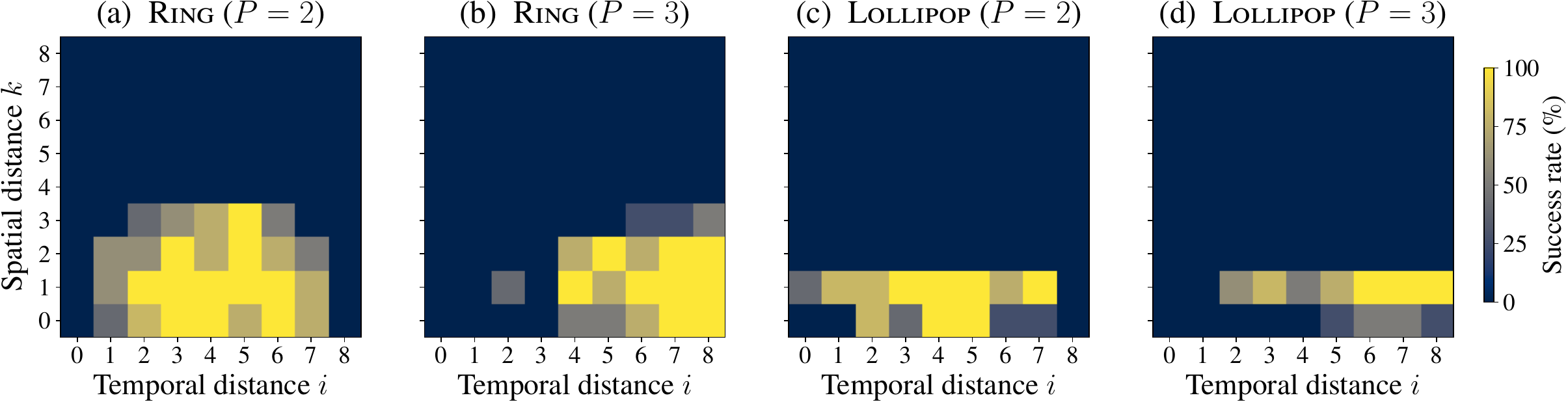}
         \caption{\textsc{\gls{tas} \gls{mptcn}}}
         \label{fig:spacetime_app_tas}
     \end{subfigure}
    \caption{Success rate (\%) of \acrshortpl{mptcn} on the \textsc{RocketMan} dataset. The tasks vary for the type of graph used (\textsc{Ring} or \textsc{Lollipop}) and size of $P$ ($2$ or $3$). The plot axes show the source neighbors distance $k$ and the temporal distance $i$.}
    \label{fig:spacetime_app}
\end{figure}

\paragraph{Training setting} We trained all models using the Adam~\citep{kingma2014adam} optimizer with an initial learning rate of $0.001$, scheduled by a cosine annealing strategy that decays the learning rate to $10^{-6}$ over the full training run. Gradients are clipped to a maximum norm of $5$ to improve stability. For synthetic experiments, we trained for a maximum of $150$ epochs with early stopping if the validation loss did not improve for $30$ consecutive epochs, using mini-batches of size $32$. To reduce computational time, we limit each epoch to the first $400$ randomly sampled batches in the experiments for \autoref{fig:memst}. For experiments on real-world datasets, we used the \gls{mae} as the loss function, trained for up to $200$ epochs with a patience of $50$ epochs, and in each epoch randomly sampled without replacement $300$ mini-batches of size $64$ from the training set.

\paragraph{Baselines}
\label{sec:baselines_appendix}
In the following, we report the hyperparameters used in the experiment for the considered architectures. Whenever possible, we relied on code provided by the authors or available within open-source libraries to implement the baselines.

\begin{description}[leftmargin=1em, itemindent=0em]
    \item[\textbf{\gls{mptcn}}] %
    We use $d=64$ hidden units and the GELU activation function~\citep{hendrycks2016gaussian} throughout all layers. As $\mpl$, we use the Diffusion Convolution operator from \citep{atwood2016diffusion}. For the real-world datasets, we compute messages with different weights from both incoming and outgoing neighbors up to $2$ hops, as done by \citet{li2018diffusion}. As the $\encl$, we upscale the input features through a random (non-trained) semi-orthogonal $d_x \times d$ matrix, such that the norm of the input is preserved. As the $\rol$, we use different linear projections for each time step in the forecasting horizon.
    \item[\textbf{\gls{gwnet}}] We used the same parameters reported in the original paper~\cite{wu2019graph}, except for those controlling the receptive field. Being \gls{gwnet} a convolutional architecture, this was done to ensure that the receptive field covers the whole input sequence.
    In particular, we used $6$ layers with temporal kernel size and dilation of $3$ for \gls{engrad} since the input window has length $24$. For the \gls{tts} implementation, we compute all message-passing operations stacked at the end of temporal processing, without interleaving any dropout or normalization (as done in the temporal part). The final representation is added through a non-trainable skip-connection to the output of temporal processing, to mimic the original \gls{tas} implementation. The two approaches share the same number of trainable parameters.
\end{description}

\section{Computational complexity}
\label{app:compcomp}
In this appendix, we analyze the computational complexity of \glspl{mptcn} with spatiotemporal message passing as defined in \autorefseq{eq:tmp_disjoint}{eq:tcn_ref}. 
A summary of computational complexities is provided in \autoref{tab:complex}, where we omit constant factors and the dependency on feature dimension $d$ for clarity.

\paragraph{Naive approach.}
We begin by analyzing the temporal component in \autoref{eq:tmp_disjoint}, where the $\tmpl$ operator is implemented as a stack of causal convolutional layers, each with kernel size $P$ as specified in \autoref{eq:tcn_ref}. Each temporal layer performs $P$ matrix-vector multiplications and thus incurs a time complexity of $\mathcal{O}(Pd^2)$. With $L_{\mathsf{T}}$ such layers, the total cost of temporal processing becomes $\mathcal{O}(L_{\mathsf{T}}Pd^2)$ per node. Since this is applied independently to all $N$ nodes across $L$ outer layers, the cumulative temporal cost is $\mathcal{O}(LL_{\mathsf{T}}NPd^2)$. Next, we consider the spatial processing in \autoref{eq:mp_disjoint}, where the $\mpl$ operator is composed of $L_{\mathsf{S}}$ message-passing layers, each defined as in \autoref{eq:mp_ref}. Assuming the message function $\phi^{(l)}$ involves a matrix multiplication, with complexity $\mathcal{O}(|\gE|d^2)$, where $|\gE|$ is the number of graph edges. Additionally, each layer applies two matrix-vector multiplications per node, one for $\Theta_{\mathsf{U}}^{(l)}$ and one for $\Theta_{\mathsf{M}}^{(l)}$. Hence, the cost of a single message-passing layer is $\mathcal{O}(|\gE|d^2 + 2Nd^2)$. 
Repeating this over $L_{\mathsf{S}}$ spatial layers, $T$ time steps, and $L$ outer layers results in a total spatial cost of $\mathcal{O}(LL_{\mathsf{S}}T(|\gE|d^2 + 2Nd^2))$.
Combining both components, the overall computational complexity of an \gls{mptcn} with number of layers $L$, $L_{\mathsf{T}}$, and $L_{\mathsf{S}}$, kernel size $P$, and hidden dimension $d$ is:
\[
\mathcal{O}\left(L\left(L_{\mathsf{T}}NPd^2 + L_{\mathsf{S}}T(|\gE|d^2 + 2Nd^2)\right)\right).
\]

\paragraph{Optimized approach.}
Encoder-decoder architectures, as the \glspl{stgnn} under study, typically require only the final representation $\vh_t^{v(L)}$ to produce the output $\hat{\vy}_t^v$. As a result, the last $\stmpl$ layer only requires performing message passing exclusively at the last time step $t$, using the embeddings $\mH_t^{(L)}$. This reduces the spatial computation at the final layer to
\(
\mathcal{O}\left(L_{\mathsf{S}}(|\gE|d^2 + 2Nd^2)\right),
\)
which is a factor of $T$ more efficient than in the naive approach.
Assuming that $PL_{\mathsf{T}} > T$, i.e., each $\stmpl$ layer has a temporal receptive field that spans the entire sequence, all preceding $\stmpl$ layers still require access to all $T$ time steps and therefore cannot be similarly optimized.
Under this assumption, the total complexity becomes:
\[
\mathcal{O}\left(LL_{\mathsf{T}}NPd^2 + \left((L-1)L_{\mathsf{S}}T + L_{\mathsf{S}}\right)(|\gE|d^2 + 2Nd^2)\right),
\]
and simplifies in the case of a \gls{tts} architecture with $L = 1$ to
\(
\mathcal{O}\left(L_{\mathsf{T}}NPd^2 + L_{\mathsf{S}}(|\gE|d^2 + 2Nd^2)\right).
\)
If we relax the assumption $PL_{\mathsf{T}} > T$, the \gls{tas} method still yields a computational gain over the naive implementation for the layers required to cover the full sequence length $T$. The remaining layers, however, incur the same cost as in the naive case. This results in a moderate speedup, albeit still significantly less efficient than the \gls{tts} approach.

\begin{table}[t]
  \centering
  \caption{Comparison of computational complexity between \gls{tts} and \gls{tas} under fixed spatial and temporal budgets ($B_S = LL_{\mathsf{S}}$, $B_T = LL_{\mathsf{T}}$), and fixed kernel size $P$, assuming that each layer's receptive field satisfies $PL_{\mathsf{T}} \geq T$. The \gls{tts} approach achieves a $T$-fold reduction in computation.}
    \label{tab:complex}
    \vskip 0.1in
  \begin{tabular}{lcc}
    \toprule
                                                                  & \textbf{\gls{tts} (L=1)} & \textbf{\gls{tas} (L>1)} \\
    \midrule
    \small
    \textbf{Naive}                                                &
    $\mathcal{O}\left(B_TNP + B_ST|\gE| \right)$ &
    $\mathcal{O}\left(B_TNP + B_ST|\gE| \right)$                                            \\
    \textbf{Optimized}                                            &
    $\mathcal{O}\left(B_TNP + B_S|\gE| \right)$         &
    $\mathcal{O}\left(B_TNP + L_{\mathsf{S}}|\gE| + (B_S - L_{\mathsf{S}})T|\gE| \right)$                                             \\
    \bottomrule
  \end{tabular}
\end{table}

\section{\textsc{TemporalNeighboursMatch}}
\label{app:neigh_match}
We propose \textsc{TemporalNeighboursMatch}, an adaptation of \textsc{NeighboursMatch} \citep{alon2021bottleneck} to the spatiotemporal setting. In \textsc{NeighboursMatch}, information is propagated from sender nodes to a root node, and the goal is to classify the root node with the label of the sender node with matching features. We extend this to a spatiotemporal setting with fixed graph topology in time and a single active time step, where only sender nodes receive non-zero features. This time step is an additional hyperparameter, akin to depth in \textsc{TreeNeighboursMatch}, where the graph topology is a tree. The goal in the \textsc{TemporalNeighboursMatch} problem remains to route the correct sender’s label~--~identified by matching features~--~to the root node at a later step. This task complements our existing synthetic benchmarks, as it emphasizes information compression, i.e., the need to retain and route specific input signals through bottlenecks.

For our experiment, we use a \textsc{TemporalTreeNeighboursMatch} controlled environment, in which the topology is a tree and is fixed, while the nonzero leaf features are placed either at the initial or the final time step.
In \autoref{tab:neigh_match}, we report test accuracy in the range $[0,1]$ for varying tree depths and a fixed number of time steps $T=4$.
Regarding spatial over-squashing, we observe the same pattern as in the original experiments of \citet{alon2021bottleneck}. Indeed, accuracy begins to drop at depth 4 and shows substantial degradation at depth 5. Notably, performance is higher at depth 4 when the relevant information appears at the start of the sequence rather than at the end, consistent with our theoretical and empirical findings.
However, the large drop in accuracy at depth 5 in both scenarios suggests that the main bottleneck is spatial rather than temporal. This indicates that the task, when combined with a tree-like topology, is limited in its ability to capture the nuances of the spatiotemporal bottleneck problem.

\begin{table}[h]
    \centering
    \vspace{-.5em}
    \caption{Test accuracy on \textsc{TemporalTreeNeighboursMatch} with varying tree depth and feature position.}
    \label{tab:neigh_match}
    \vskip 0.1in
    \setlength{\tabcolsep}{6pt}
    \setlength{\aboverulesep}{0pt}
    \setlength{\belowrulesep}{0pt}
    \renewcommand{\arraystretch}{1.2}
    \begin{tabular}{|c|ccc|}
        \cmidrule[.8pt]{2-4}
         \multicolumn{1}{c|}{} & \multicolumn{3}{c|}{\textsc{Tree Depth}}\\
        \toprule
         \textsc{Features Position} & 3 & 4 & 5 \\
        \toprule
         \textbf{First step} & 1.00 {\small $\pm$0.00} & 1.00 {\small $\pm$0.00} & \cellcolor{gray!18}{0.09 {\small $\pm$0.01}} \\
        \textbf{Last step} & 1.00 {\small $\pm$0.01} & \cellcolor{gray!8}{0.73 {\small $\pm$0.46}} & \cellcolor{gray!20}{0.11 {\small $\pm$0.08}} \\
        \bottomrule
    \end{tabular}
\end{table}


\end{document}